\UseRawInputEncoding
\documentclass[10pt,journal,compsoc]{IEEEtran}

\ifCLASSOPTIONcompsoc
  \usepackage[nocompress]{cite}
\else
  \usepackage{cite}
\fi

\ifCLASSINFOpdf
\else
\fi
\usepackage{graphicx}
\usepackage{amsmath,amsfonts,amssymb,amsthm} 
\usepackage{algorithm}
\usepackage{algorithmic}
\usepackage{subfigure}
\usepackage{booktabs}
\renewcommand{\Re}{\mathbb{R}}
\newcommand{\Id}{\mathtt{I}}

\newcommand{\diag}[1]{\ensuremath{\mathrm{diag}\left(#1\right)}}

\newcommand{\rank}[1]{\ensuremath{\mathrm{rank}\left(#1\right)}}
\newcommand{\norm}[1]{\ensuremath{\left\| #1 \right\|}}

\newtheorem{prob}{Problem}[section]
\newtheorem{theo}{Theorem}[section]
\newtheorem{defn}{Definition}[section]
\newtheorem{lemm}{Lemma}[section]

\begin{document}

\title{Time Series Forecasting via Learning Convolutionally Low-Rank Models}

\author{Guangcan Liu,~\IEEEmembership{Senior Member,~IEEE}
\IEEEcompsocitemizethanks{\IEEEcompsocthanksitem G. Liu is with the School of Automation, Southeast University, Nanjing, China 210018. Email: gcliu1982@gmail.com.}
\thanks{Copyright (c) 2017 IEEE. Personal use of this material is permitted.  However, permission to use this material for any other purposes must be obtained from the IEEE by sending a request to pubs-permissions@ieee.org.}}
	
\IEEEtitleabstractindextext{
\begin{abstract}
Recently, Liu and Zhang~\cite{liu:arxiv:2019} studied the rather challenging problem of time series forecasting from the perspective of compressed sensing. They proposed a no-learning method, named Convolution Nuclear Norm Minimization (CNNM), and proved that CNNM can exactly recover the future part of a series from its observed part, provided that the series is convolutionally low-rank. While impressive, the convolutional low-rankness condition may not be satisfied whenever the series is far from being seasonal, and is in fact brittle to the presence of trends and dynamics. This paper tries to approach the issues by integrating a learnable, orthonormal transformation into CNNM, with the purpose for converting the series of involute structures into regular signals of convolutionally low-rank. We prove that the resultant model, termed Learning-Based CNNM (LbCNNM), strictly succeeds in identifying the future part of a series, as long as the transform of the series is convolutionally low-rank. To learn proper transformations that may meet the required success conditions, we devise an interpretable method based on Principal Component Pursuit (PCP). Equipped with this learning method and some elaborate data argumentation skills, LbCNNM not only can handle well the major components of time series (including trends, seasonality and dynamics), but also can make use of the forecasts provided by some other forecasting methods; this means LbCNNM can be used as a general tool for model combination. Extensive experiments on 100,452 real-world time series from Time Series Data Library (TSDL) and M4 Competition (M4) demonstrate the superior performance of LbCNNM.
\end{abstract}
\begin{IEEEkeywords}
compressed sensing, sparsity and low-rankness, dictionary learning, time series forecasting, model combination, Fourier transform, coherence.
\end{IEEEkeywords}}

\maketitle

\IEEEdisplaynontitleabstractindextext
\IEEEpeerreviewmaketitle

\section{Introduction}
\IEEEPARstart{T}{ime} series forecasting, the problem of making forecasts for future based on historical observations, has found tremendous significance in many areas, ranging from machine learning, statistics, signal processing and pattern recognition to econometrics, finance, meteorology and geology. In a typical setting, univariate time series forecasting, one is given a sequence of length $m-h$ and is asked to predict the next $h$ unseen values ($m>h$), where $h$ and $m$ are called respectively \emph{forecast horizon} and \emph{model size}. While its definition is quite brief, the problem is incredibly difficult, remaining challenging after decades of research and still demanding new scenarios~\cite{Pablo:arxiv:2020}. Very recently, Liu and Zhang~\cite{liu:arxiv:2019} studied this problem from the viewpoint of \emph{compressed sensing}~\cite{donoho:tit:2006,Candes:2008:spm}, and they suggested
to regard tensor-valued time series forecasting as a special case of \emph{tensor completion with arbitrary sampling}. In the framework of univariate series, the problem is described as follows:
\begin{prob}[Vector Completion with Arbitrary Sampling]\label{pb:seqc}
Let $\mathbf{y}\in\mathbb{R}^{m}$ be a vector that represents some time series of length $m$. Let $\Omega$ be a sampling set consisting of positive integers \textbf{arbitrarily} selected from the range between 1 and $m$, and let $\mathcal{P}_{\Omega}$ be the orthogonal projection onto the subspace of vectors supported on $\Omega$. Given $\mathcal{P}_{\Omega}(\textbf{y})$, which is $\textbf{y}$ by zeroing out the values whose indices are not included in $\Omega$, the goal is to recover $\textbf{y}$.
\end{prob}

The setup of \emph{arbitrary} sampling pattern is powerful, not only including series forecasting as a special case with $\Omega=\{1,\cdots,m-h\}$ but also bringing the benefits that the historical part of a series is allowed to be incomplete---this is useful in many situations~\cite{wu:ICINIS:2015,rnn:missing:2018}. To address Problem~\ref{pb:seqc}, Liu and Zhang~\cite{liu:arxiv:2019} proposed a convex program called Convolution Nuclear Norm Minimization (CNNM):
\begin{align}\label{eq:cnnm:exact}
&\min_{\mathbf{x}\in\mathbb{R}^m} \norm{\mathcal{A}_k(\mathbf{x})}_{*},\quad\textrm{s.t.}\quad\mathcal{P}_{\Omega}(\mathbf{x}-\mathbf{y})=0,
\end{align}
where $\mathcal{A}_k(\cdot)$ is a linear map from $\mathbb{R}^m$ to $\mathbb{R}^{m\times{}k}$ such that $\mathcal{A}_k(\mathbf{x})$ produces the \emph{convolution matrix} (see Section~\ref{sec:convmatrix}) of $\mathbf{x}$, $\|\cdot\|_*$ is the \emph{nuclear norm}~\cite{phd_2002_nuclear} of a matrix, $k=\alpha{}m$ is the kernel size used in defining $\mathcal{A}_k(\cdot)$, and $1.5h/m\leq\alpha\leq{1}$ is a parameter. It was proven in~\cite{liu:arxiv:2019} that the target $\mathbf{y}$ we want to recover is the unique minimizer to the problem in~\eqref{eq:cnnm:exact}, provided that the sampling complexity, $\rho_0 = \mathrm{card}(\Omega)/m$ ($\mathrm{card}(\cdot)$ is the cardinality of a set), satisfies a condition as follows:
\begin{align*}
\rho_0 >1-\mathcal{O}(1/r),
\end{align*}
where $r$, called \emph{convolution rank}, is simply the rank of the convolution matrix $\mathcal{A}_k(\mathbf{y})$. At the first glance, the sampling complexity of order $1-\mathcal{O}(1/r)$ is seemingly ``awkward'' and much worse than the ``beautiful'' bound of $\mathcal{O}(r(\log{m})^2/m)$ proven in many matrix completion papers such as~\cite{Chen:2015:tit}. Yet, as clarified by~\cite{liu:arxiv:2019}, $\mathcal{O}(r(\log{m})^2/m)$ is not possible while applying CNNM to forecasting, in which the sampling pattern is deterministic rather than random, and $1-\mathcal{O}(1/r)$ is indeed pretty tight---probably not optimal though. The complexity of $1-\mathcal{O}(1/r)$ is meaningful for forecasting, as it implies that the forecast horizon can be as large as $h=\mathcal{O}(m/r)$. This assertion, however, relies on a critical premise; that is, $r$ has to be smaller than $\mathcal{O}(m)$, i.e., the target $\mathbf{y}$ is \emph{convolutionally low-rank}, which may not happen whenever seasonality is absent from $\mathbf{y}$. In fact, the convolutional low-rankness condition required by CNNM may not stand the trends that are ubiquitous in real-life time series, not even the elusive dynamics such as knots and abrupt changes. We shall try to break through these limits via learning a proper representation for the target $\mathbf{y}$.

Notice that, for any vector $\mathbf{z}\in\mathbb{R}^m$, there always exists an orthonormal (i.e., column-wisely orthogonal) matrix $A\in\mathbb{R}^{q\times{}m}$ ($q\geq{}m$) such that $A\mathbf{z}$ is convolutionally low-rank, namely $\rank{\mathcal{A}_k(A\mathbf{z})}\leq2$ (the proof can be found in Section~\ref{sec:pca}). So, in order to make a series whose convolution matrix is high-rank or even full-rank to be convolutionally low-rank, it is suitable to consider the linear, orthonormal transformations from $\mathbb{R}^m$ to $\mathbb{R}^q$, resulting in the following convex program:
\begin{align}\label{eq:lbcnnm:exact}
&\min_{\mathbf{x}\in\mathbb{R}^m} \norm{\mathcal{A}_k(A\mathbf{x})}_{*},\quad\textrm{s.t.}\quad\mathcal{P}_{\Omega}(\mathbf{x}-\mathbf{y})=0,
\end{align}
where $A\in\mathbb{R}^{q\times{}m}$ that satisfies $A^TA=\Id_m$ is learnt from some training samples in advance, $\Id_m$ is the $m\times{}m$ identity matrix, the convolution matrix $\mathcal{A}_k(A\mathbf{x})$ is of size $q\times{}k$, and the kernel size can be set as $k=\beta{}q$ ($0<\beta\leq1$ is taken as a parameter). Hereafter, the above program is referred to as Learning-Based CNNM (LbCNNM). We will prove that LbCNNM exactly recovers the target $\mathbf{y}$ from its observed part $\mathcal{P}_{\Omega}(\mathbf{y})$, as long as
\begin{align*}
\rho_0 >1-\mathcal{O}(1/r),
\end{align*}
where $r=\rank{\mathcal{A}_k(A\mathbf{y})}$ is the convolution rank of the transformed signal $A\mathbf{y}$. Clearly, the key for LbCNNM to succeed in solving Problem~\ref{pb:seqc} is to learn proper transform matrix $A$ such that $\mathcal{A}_k(A\mathbf{y})$ is fairly low-rank.

To actualize the learning of $A$, first of all, we need a data matrix, denoted as $Y\in\mathbb{R}^{m\times{}n}$, that contains $n$ training samples. Usually, $Y$ is constructed based on a training sequence, $\tilde{\mathbf{y}}\in\mathbb{R}^{l}$, that is somehow relevant to the target $\mathbf{y}$. Indeed, obtaining a good $Y$ is the most critical premise for LbCNNM to succeed, and the construction procedure must be elaborately designed such that the temporal structures of time series are handled properly.\footnote{Except the procedure for constructing $Y$, LbCNNM is more like a general representation model than a specific time series forecasting method.} The approach for constructing $Y$ is not unique, and there are several ways to fulfil the task, as we will show in Section~\ref{sec:gematrix:pcp} and Section~\ref{sec:argumentation}. Yet, a firm principle should be respected is that $Y$ is better to be as low-rank as possible. At this moment, we temporarily assume that a qualified data matrix $Y$ has already been obtained.

Given some data matrix $Y\in\mathbb{R}^{m\times{}n}$, we will elucidate that the learning of $A$ that leads to convolutional low-rankness can be done by finding another orthonormal matrix $B\in\mathbb{R}^{q\times{}m}$ ($B^TB=\Id_m$) such that $BY$ is sparse. The latter is closely related to the problem of learning a possibly overcomplete dictionary $B^T\in\mathbb{R}^{m\times{}q}$ that expresses each sample as a linear combination of few dictionary atoms, i.e., the well-known \emph{sparsely-used dictionary learning}~\cite{Aharon:tsp:2006,michal:laa:2006,john:colt:2012,sun:tit:2017}. Nevertheless, there is a key difference; that is, $Y=B^TE$ with $E$ being sparse cannot really ensure that $BY$ is sparse, thereby the existing dictionary learning methods may not meet our demands. Thus we devise a method based on Principal Component Pursuit (PCP)~\cite{Candes:2009:JournalACM}. In short, our method first decomposes $Y$ into a low-rank component and a sparse component, and then obtains $B$ by solving some sparse minimization problem designed to encode both the low-rank and the sparse components in a joint fashion.

Depending on what kind of data matrix $Y$ is used to learn the transformation $A$, LbCNNM not only can handle well the major components of time series, including seasonality, trends and dynamics, but also can take advantage of the forecasts provided by some existing methods such as Average, Drift, Least Square Regression (LSR), CNNM and Exponential Smoothing (ExpS). These abilities enable the possibility for LbCNNM to cope with a wide variety of time series. Extensive experiments on 100,452 series from TSDL~\cite{TSDL:dataset} and M4~\cite{M4:ijf:2020} confirm the effectiveness of LbCNNM. To summarize, the main contributions of this paper include:
\begin{itemize}
\item[$\bullet$]We establish a novel method, termed LbCNNM, for time series forecasting. Besides its strengths in handling the trends, seasonality and dynamics, LbCNNM also owns the ability of combining multiple weak models into a strong one. As acknowledged by many studies, e.g.,~\cite{M4:ijf:2020,Pablo:arxiv:2020}, \emph{model combination} is one of the most effective strategies for improving the accuracy of time series forecasting.
\item[$\bullet$]We prove that the recovery property of CNNM, which is a fixed model, hold for its learning-based extension, LbCNNM. This may widen dramatically the application scope of convolutionally low-rank models.
\item[$\bullet$] This study reveals that the most crucial problem in overcomplete dictionary learning is probably the coherence issue---this might be of independent interest. In addition, the proposed LbCNNM is indeed a general method for sparse representation learning, and thus might have values outside the scope of time series forecasting.
\end{itemize}

The rest of this paper is organized as follows. Section~\ref{sec:note_pre} summarizes the mathematical notations used throughout the paper and introduces some background knowledge as well. Section~\ref{sec:meth_result} presents the technical details of the proposed methods. Section~\ref{sec:proofs} shows the mathematical proofs of the proposed theorems. Section~\ref{sec:exp} consists of empirical results and Section~\ref{sec:conclusion} concludes this paper.
\section{Notations and Preliminaries}\label{sec:note_pre}
\subsection{Summary of Main Notations}
Matrices and vectors are denoted as capital letters and bold lowercase letters, respectively. Single numbers are denoted by either lowercase or Greek letters. Four types of matrix norms are used frequently: the Frobenius norm $\|\cdot\|_F$ defined as the square root of the sum of the squares of the entries of a matrix, the $\ell_1$ norm $\|\cdot\|_1$ given by the sum of the absolute values of the matrix entries, the operator norm $\|\cdot\|$ defined as the largest singular value, and the nuclear norm $\|\cdot\|_*$ calculated as the sum of singular values. For a vector $\mathbf{z}$, $\|\mathbf{z}\|_1$, $\|\mathbf{z}\|_2$ and $\|\mathbf{z}\|_0$ are its $\ell_1$, $\ell_2$ and $\ell_0$ norms, respectively. Letters $U$ and $V$ are reserved respectively for the left and right singular vectors of a matrix. The orthogonal projection onto the column space is denoted by $\mathcal{P}_U$ and given by $\mathcal{P}_U(Z)=UU^TZ$, and similarly for the row space $\mathcal{P}_V(Z)=ZVV^T$. The same notation is also used to represent the images of orthogonal projections, e.g., $Z\in\mathcal{P}_U$ means that $\mathcal{P}_U(Z)=Z$. The symbol $\otimes$ is reserved for the Kronecker product, $\Id_d$ is the identity matrix with size $d\times{}d$, and $[d]$ denotes the first $d$ positive integers, i.e., $[d]=\{1,\cdots,d\}$.
\subsection{Discrete Fourier Transform}\label{sec:dft}
Discrete Fourier Transform (DFT) is a fundamental tool for data analysis. From the view of linear algebra, the DFT of a vector $\mathbf{z}\in\mathbb{R}^w$ is simply
\begin{align*}
\mathcal{F}(\mathbf{z}) = F\mathbf{z},
\end{align*}
where $F\in\mathbb{C}^{w\times{}w}$ is a complex-valued, symmetric matrix obeying $F^HF = FF^H=w\Id_w$. The transform matrix, $F$, is data-independent and can be generated solely based on $w$. Denote by $F_1\in\mathbb{R}^{w\times{}w}$ and $F_2\in\mathbb{R}^{w\times{}w}$ the real and imaginary components of $F$, respectively. It is easy to see that $F_1F_1^T+F_2F_2^T = w\Id_w$ and $F_1^TF_1+F_2^TF_2 = w\Id_w$. Take the skinny Singular Value Decomposition (SVD) of these two matrices:
\begin{align}\label{eq:dft:u1v1}
F_1=\sqrt{w}U_1V_1^T \quad\mathrm{and}\quad F_2=\sqrt{w}U_2V_2^T,
\end{align}
where $\rank{U_1}+\rank{U_2} = w$. The first projector, $U_1U_1^T$, is indeed the sum of a diagonal matrix and an anti-diagonal matrix, thus preserving sparsity; namely, $\|U_1U_1^T\mathbf{z}\|_0\leq{}2\|\mathbf{z}\|_0$, $\forall{\mathbf{z}}$. The second projector, $U_2U_2^T=\Id_w-U_1U_1^T$, also has this property. These elementary knowledge is useful for learning the transform matrix $A$ in LbCNNM. In particular, we need to construct two orthogonal matrices as follows:
\begin{align}\label{eq:dft:ufvf}
U_F = [U_1,U_2]\in\mathbb{R}^{w\times{}w}\quad\textrm{and}\quad{}V_F=[V_1,V_2]\in\mathbb{R}^{w\times{}w},
\end{align}
where $[\cdot,\cdot]$ is to concatenate two matrices together horizontally.
\subsection{Convolution Matrix}\label{sec:convmatrix}
The definition of (discrete) convolution has many variants, depending on which \emph{boundary condition} is used. The same as in~\cite{liu:arxiv:2019}, this paper considers the \emph{circular convolution}, i.e., convolution with the \emph{circulant boundary condition}~\cite{mamdouh:sip:2012}. For $\mathbf{a}=[a_1,\cdots,a_w]^T\in\mathbb{R}^w$ and $\mathbf{b}=[b_1,\cdots,b_k]^T\in\mathbb{R}^k$ ($k\leq{}w$), the procedure of convoluting them into $\mathbf{c}=\mathbf{a}\star{}\mathbf{b}\in\mathbb{R}^w$ ($\star$ is the convolution operator) is given by
\begin{align*}
c_i = \sum_{j=1}^k a_{i-j}b_{j}, \forall{}i\in[w],
\end{align*}
where $c_i$ is the $i$th entry of $\mathbf{c}$. Here, it is assumed that $a_{i-j} = a_{i-j+w}$ for $i\leq{}j$; this is why the operator is circular. In general, circular convolution is a linear operator and can be therefore expressed as
\begin{align*}
\mathbf{a}\star{}\mathbf{b} = \mathcal{A}_k(\mathbf{a})\mathbf{b},
\end{align*}
where $\mathcal{A}_k(\mathbf{a})\in\mathbb{R}^{w\times{}k}$ is the \emph{convolution matrix} of $\mathbf{a}$ with respect to kernel size $k$. In fact, the convolution matrix is a truncated version of the circular matrix:
\begin{align}\label{eq:convmtx}
\mathcal{A}_k(\mathbf{a}) &= \left[\begin{array}{cccc}
a_1 & a_w&\cdots&a_{w-k+2}\\
a_2 & a_1&\cdots&a_{w-k+3}\\
\vdots&\vdots&\vdots&\vdots\\
a_w & a_{w-1}&\cdots&a_{w-k+1}
\end{array}\right]\\\nonumber
&=[\mathbf{a},T\mathbf{a},\cdots,T^{k-1}\mathbf{a}],
\end{align}
where $T\in\mathbb{R}^{w\times{}w}$ is a permutation matrix such that $TZ$ circularly shifts all the rows of $Z$ by one position, $\forall{Z}\in\mathbb{R}^{w\times{}w_1}$, $w_1\geq1$. Whenever $k=w$, the produced convolution matrix, $\mathcal{A}_w(\mathbf{a})\in\mathbb{R}^{w\times{}w}$, is diagonalized by DFT~\cite{MonC2008}. In this case, the so-called convolution nuclear norm falls back to the $\ell_1$ norm of the Fourier transform. That is, $\|\mathcal{A}_w(\mathbf{a})\|_* = \|\mathcal{F}(\mathbf{a})\|_1$, $\forall{}\mathbf{a}\in\mathbb{R}^w$, where $\mathcal{F}$ is the DFT operator.

Within the bounds of time series, there are certain reasons for the convolution matrix to be low-rank or approximately so. For a periodic series $\mathbf{z}\in\mathbb{R}^m$ with period $p$, i.e., $z_i = z_{i+p},\forall{}i$, its convolution matrix $\mathcal{A}_k(\mathbf{z})\in\mathbb{R}^{m\times{}k}$ obeys
\begin{align*}
\rank{\mathcal{A}_k(\mathbf{z})}\leq{}p, \textrm{ provided that } m=cp,
 \end{align*}
where $c$ is some positive integer. Note here that the condition of $m=cp$ is necessary in general cases, and thus the model size $m$ in CNNM is actually a hyper-parameter needs be estimated carefully. Besides periodicity, the local continuity of a series can also lead to convolutional low-rankness in an approximate manner, as in this case the convolution matrix may contain many near-duplicate columns and rows. Hence, CNNM, which is a fixed prediction function, owns certain ability of making forecasts. However, since periodicity has been assumed inherently in circular convolution, the performance of CNNM may drop dramatically whenever the series is very different from being seasonal. For this reason, CNNM is not applicable to many types of time series, e.g., a straight line with large slope---hinting that the trends exist widely in time series may violate the convolutional low-rankness condition required by CNNM.
\subsection{Generation Matrix}\label{sec:gematrix:pcp}
To learn a proper transform matrix $A$ such that $A\mathbf{y}$ is convolutionally low-rank, in general, it is necessary to obtain some training samples that are somehow relevant to the target $\mathbf{y}$. The quality of training data, as aforementioned, is crucial for LbCNNM to succeed. In this work, we adopt a straightforward approach as follows. Consider the standard forecasting setup that is to predict the future values of a series based on its historical part. Let's re-denote the historical part as $\tilde{\mathbf{y}}=[\tilde{y}_1,\cdots,\tilde{y}_l]\in\mathbb{R}^{l}$ and treat $\tilde{\mathbf{y}}$ as a training sequence. Then we can set $m\leq{}l$ (so the model size $m$ is a hyper-parameter in this setup), form the observed part of $\mathbf{y}$ by taking the last $m-h$ observations from $\tilde{\mathbf{y}}$, and construct a matrix $G_0\in\mathbb{R}^{m\times{}n_0}$ ($n_0=l-m+1$) as
\begin{align}\label{eq:gemtx}
G_0 = \left[\begin{array}{cccc}
\tilde{y}_1 & \tilde{y}_{2}&\cdots&\tilde{y}_{l-m+1}\\
\tilde{y}_2 & \tilde{y}_{3}&\cdots&\tilde{y}_{l-m+2}\\
\vdots&\vdots&\vdots&\vdots\\
\tilde{y}_m & \tilde{y}_{m+1}&\cdots&\tilde{y}_{l}
\end{array}\right],
\end{align}
which is called the \emph{generation matrix}, each column of which is a training sample. Since the columns of $G_0$ and the target $\mathbf{y}$ are extracted from the same series with the same temporal order, it is reasonable to believe that they may have good chance to be statistically relevant.

As one can see, the generation matrix has a structure very similar to the convolution matrix. But there are still some notable differences. First of all, for a periodic series $\tilde{\mathbf{y}}$ with period $p$, the rank of its generation matrix is always bounded from above by $p$:
\begin{align*}
\rank{G_0}\leq{}p,\forall{}m,
\end{align*}
which means the generation matrix is rank-deficient as long as $\min(m,n_0)>p$. For any linear series, $\tilde{y}_i = ai+b, i\in[l],\forall{}a,b$, it is easy to see that its generation matrix has a rank of at most 2. As a consequence, provided that $\tilde{\mathbf{y}}$ can be decomposed into the sum of $c_1$ periodic signals (with period at most $p$) and $c_2$ ($c_2\leq2$) lines, we have
\begin{align*}
\rank{G_0}\leq{}pc_1 + 2c_2, \forall{}m.
\end{align*}
Whenever the premise holds only in an approximate sense; namely, $\tilde{\mathbf{y}}=\bar{\mathbf{y}} + \mathbf{n}$ with $\|\mathbf{n}\|_2\leq{}\epsilon$ and $\bar{\mathbf{y}}$ obeying the condition,
we also have
\begin{align*}
\|G_0 - \bar{G}_0\|_F\leq\sqrt{\min(m,n_0)}\epsilon\textrm{  with  }\rank{\bar{G}_0}\leq{}pc_1 + 2c_2,
\end{align*}
where $\bar{G}_0$ is the generation matrix extracted from $\bar{\mathbf{y}}$.

The above deductions indicate that the seasonality and trends, which are two major components of time series, generally lead to low-rankness. Another important component of time series is the dynamics, including kinks, knots, sudden changes, and so no. Such elusive structures may not necessarily lead to low-rankness, and representation learning is unlikely to be able to \emph{fully} convert them into low-rankness either.\footnote{For any kinds of dynamics, as aforemention in Introduction, there does exist a transformation $A$ that can transform them into something of convolutionally low-rank. However, the existence of $A$ does not necessarily mean that it is practical to learn the desired $A$ in any cases.} Surprisingly, using the learning method proposed in Section~\ref{sec:pcp}, certain kind of dynamics that have trackable structures could be represented in a low-rank fashion.

In most elucidation of this paper, the generation matrix $G_0$ will be used to construct the data matrix $Y$ that is needed by the learning of the transform matrix $A$. However, it is particularly worth mentioning that $G_0$ is never the unique choice, and there are indeed a great many blueprints for designing $Y$, as we will exemplify in Section~\ref{sec:argumentation}.
\subsection{Robust Measures of Low-Rankness} The property of low-rankness is a reliable clue for prejudging the performance of LbCNNM, and we often need to measure how close a matrix $Z$ is to be low-rank. To this end, the most naive approach is to firstly obtain the vector of all (zero and non-zero) singular values of $Z$, denoted as $\mathbf{s}$, and then calculate the $\ell_0$ norm of $\mathbf{s}$. Such a naive approach, however, seldom works in practice, as very often the singular value vector $\mathbf{s}$ is heavily-tailed rather than strictly sparse. Hence, we consider instead another two robust measures of sparsity: entropy and Gini~\cite{gini:tit:2009}. The entropy of the spectrum of $Z$, called \emph{spectral entropy} and denoted as $\mathrm{SpEnt}(Z)$, is calculated as follows:
\begin{align}\label{eq:ent}
\mathrm{SpEnt}(Z) = -\frac{\sum_{i=1}^{n_b}p_i\log_2{p_i}}{\log_2{n_b}}\in[0, 1],
\end{align}
where the values in $\mathbf{s}$ are partitioned into $n_b$ bins (we set $n_b=5$), and $p_i$ is the probability for the values in $\mathbf{s}$ to fall into the $i$th bin. The smaller the spectral entropy is, the more close the matrix $Z$ is to be low-rank. The \emph{spectral Gini} of $Z$, denoted as $\mathrm{SpGini}(Z)$, is simply the Gini of its singular value vector $\mathbf{s}$:
\begin{align}\label{eq:gini}
\mathrm{SpGini}(Z) = \mathrm{Gini}(\mathbf{s})\in[0, 1].
\end{align}
The detailed formula for calculating the Gini of a vector can be found in~\cite{gini:tit:2009}. Unlike the spectral entropy, which is an increasing function of rank, the spectral Gini is larger means that the matrix is more close to be low-rank.
\subsection{Coherence}\label{sec:conv:coherence}
\noindent\textbf{Standard Coherence:} The concept of \emph{coherence}~\cite{Candes:2009:math} has been being widely used for theoretical analysis. Indeed, coherence is not just a tool for making proofs, but instead touches some fundamental aspects of learning, e.g., the intrinsic structures of data~\cite{liu:tpami:2016b,liu:tsp:2016,liu:2014:nips}. For a rank-$r$ matrix $Z\in\mathbb{R}^{q\times{}k}$, let its skinny SVD be $Z=U\Sigma{}V^T$ with $\Sigma\in\mathbb{R}^{r\times{}r}$. Then there are two coherence parameters, $\mu_1$ and $\mu_2$, for characterizing some properties of $Z$:
\begin{align}\label{eq:coherence}
\mu_1(Z) = \frac{q}{r}\max_{i\in[q]}\|U^T\mathbf{e}_i\|_2^2,\quad \mu_2(Z) = \frac{k}{r}\max_{j\in[k]}\|V^T\tilde{\mathbf{e}}_j\|_2^2,
\end{align}
where $\mathbf{e}_i$ and $\tilde{\mathbf{e}}_j$ are the $i$th and $j$th standard bases of $\mathbb{R}^q$ and $\mathbb{R}^k$, respectively. By definition, $1\leq\mu_1(Z)\leq{}q$ and $1\leq\mu_2(Z)\leq{}k$. Here, the lower bound is attained if $Z$ is a nonzero constant matrix (whose elements are the same), and the upper bounds are achieved when $Z$ has only one nonzero entry.\\

\noindent\textbf{Convolution Coherence:} The so-called \emph{convolution coherence}~\cite{liu:arxiv:2019}, denoted as $\hat{\mu}(\cdot)$, of a vector is simply the coherence of its convolution matrix:
\begin{align}\label{eq:convcoherence}
\hat{\mu}_1(\mathbf{z}) = \mu_1(\mathcal{A}_k(\mathbf{z})),\quad\hat{\mu}_2(\mathbf{z}) = \mu_2(\mathcal{A}_k(\mathbf{z})), \quad\forall{}\mathbf{z}\in\mathbb{R}^q.
\end{align}
The definitions in~\eqref{eq:coherence} and \eqref{eq:convcoherence} immediately lead to $1\leq\hat{\mu}_1(\mathbf{z})\leq{}q$ and $1\leq\hat{\mu}_2(\mathbf{z})\leq{}k$. The minimum convolution coherence can be achieved by many examples, e.g., $\mathbf{z}$ is a nonzero constant vector, $\mathbf{z}$ is the sine series with $z_i=\sin(2i\pi/q)$, etc. When $k=1$ (i.e., the trivial kernel), the maximum is attained if $\mathbf{z}$ is a standard basis in $\mathbb{R}^q$. In the general cases with $k>1$, however, it is hard, if not impossible, to find an example that can attain the maximum---this is indeed beneficial. To help readers understand the convolution coherence, we would like to establish some data-dependent bounds. Denote by $\nu$ the \emph{condition number} of the convolution matrix of $\mathbf{z}$, i.e., $\nu$ is the ratio of the largest singular value of $\mathcal{A}_k(\mathbf{z})$ to its smallest nonzero singular value. Based on the fact that each column or row of $\mathcal{A}_k(\mathbf{z})$ has an $\ell_2$ norm of at most $\|\mathbf{z}\|_2$, it can be proven that
\begin{align*}
1\leq{}\hat{\mu}_1(\mathbf{z})\leq{}q\nu^2/k \textrm{ and } 1\leq{}\hat{\mu}_2(\mathbf{z})\leq{}\nu^2,
\end{align*}
where the two upper bounds are both attained if $\mathbf{z}$ is a standard basis of $\mathbb{R}^q$, $\forall{}k\in[q]$.
\begin{proof}Let the skinny SVD of $\mathcal{A}_k(\mathbf{z})$ be $\mathcal{A}_k(\mathbf{z})=U\Sigma{}V^T$, where $\Sigma=\diag{\sigma_1,\cdots,\sigma_r}$, and $\sigma_1\geq\cdots\geq\sigma_r>0$ are the nonzero singular values of $\mathcal{A}_k(\mathbf{z})$. Considering the $\ell_2$ norm of the $j$th column of $\mathcal{A}_k(\mathbf{z})$, we have
\begin{align*}
\|\mathbf{z}\|_2^2=\|U\Sigma{}V^T\tilde{\mathbf{e}}_j\|_2^2=\|\Sigma{}V^T\tilde{\mathbf{e}}_j\|_2^2\geq{}\sigma_r^2\|V^T\tilde{\mathbf{e}}_j\|_2^2,
\end{align*}
which gives that
\begin{align*}
\|V^T\tilde{\mathbf{e}}_j\|_2^2\leq\frac{\|\mathbf{z}\|_2^2}{\sigma_r^2}=\frac{\sum_{i=1}^r\sigma_i^2}{k\sigma_r^2}\leq\frac{r}{k}\nu^2,
\end{align*}
from which it follows that $\hat{\mu}_2(\mathbf{z})\leq{}\nu^2$. Similarly, it can be proven that $\hat{\mu}_1(\mathbf{z})\leq{}q\nu^2/k$.
\end{proof}

\noindent\textbf{Generalized Convolution Coherence:} The standard coherence in~\eqref{eq:coherence} is specific to the standard bases. To figure out the success conditions of LbCNNM, it is necessary to access the concept of \emph{generalized coherence}~\cite{gross:tit:2011} with respect to any bases. Notice that, for $Z\in\mathbb{R}^{q\times{}k}$ with skinny SVD $Z=U\Sigma{}V^T$, $\mathcal{P}_U$ is indeed an orthogonal projection onto a subspace of vectors in $\mathbb{R}^{qk}$---consider the vectorized form of $\mathcal{P}_U$ with the corresponding projector being $\Id_k\otimes{}UU^T$. Then it is easy to see that the coherence defined in~\eqref{eq:coherence} is to measure the $\ell_2$ norms of the projections onto $\Id_k\otimes{}UU^T$ of the bases in
\begin{align*}
\Id_k\otimes{}\Id_q = \left[\begin{array}{ccc}
\Id_q & &\\
&\ddots&\\
& &\Id_q
\end{array}\right]\in\mathbb{R}^{qk\times{}qk},
\end{align*}
in which each column is a standard basis of $\mathbb{R}^{qk}$. Now, turn to the model of LbCNNM in~\eqref{eq:lbcnnm:exact}. In consideration of the circular convolution with kernel size $k$, the orthonormal transformation $A\in\mathbb{R}^{q\times{}m}$ gives a basis matrix as follows (note that the convolution matrix of $A\mathbf{z}$ is $[A\mathbf{z},TA\mathbf{z},\cdots,T^{k-1}A\mathbf{z}]$, $\forall{}\mathbf{z}$):
\begin{align}\label{eq:covbases}
\left[\begin{array}{cccc}
\tilde{A} & & &\\
&T\tilde{A}&&\\
&&\ddots&\\
& &&T^{k-1}\tilde{A}
\end{array}\right]\in\mathbb{R}^{qk\times{}qk},
\end{align}
where $\tilde{A}=[A, A^\bot]\in\mathbb{R}^{q\times{}q}$ is an orthogonal matrix formed by concatenating $A$ and its orthogonal complement $A^\bot$ together horizontally, and $T\in\mathbb{R}^{q\times{}q}$ is the permutation matrix defined in~\eqref{eq:convmtx}. It is worth noting that, the columns of $A$ are \emph{incomplete} as the bases for $\mathbb{R}^q$---though its rows are \emph{overcomplete} for $\mathbb{R}^m$, thereby it is necessary to do expansion using $A^\bot$. With these notations, the first \emph{generalized convolution coherence} is defined as follows.
\begin{defn}[Generalized Convolution Coherence]\label{defn:geconvcoherence}Let $\mathcal{A}_k(A\mathbf{z})\in\mathbb{R}^{q\times{}k}$ be the rank-$r$ convolution matrix of the transformed signal $A\mathbf{z}\in\mathbb{R}^{q}$ of a vector $\mathbf{z}\in\mathbb{R}^{m}$, where $A\in\mathbb{R}^{q\times{}m}$ obeys $A^TA=\Id_m$. Let the skinny SVD of $\mathcal{A}_k(A\mathbf{z})$ be $U\Sigma{}V^T$, where $\Sigma\in\mathbb{R}^{r\times{}r}$. Then the first generalized convolution coherence of $\mathbf{z}$, denoted as $\mu_A(\mathbf{z})$, is defined as
\begin{align*}
\mu_A(\mathbf{z}) = \frac{q}{r}\max_{i\in[m],j\in[k]}\|U^TT^{j-1}A\mathbf{e}_i\|_2^2,
\end{align*}
where $\mathbf{e}_i$ is the $i$th standard basis of $\mathbb{R}^m$, and $T$ is a permutation matrix defined in the same way as in~\eqref{eq:convmtx}.
\end{defn}
As one may have noticed, the generalized convolution coherence is nothing more than a quantity that measures the $\ell_2$ norms of the projections onto $\Id_k\otimes{}UU^T$ of the bases in~\eqref{eq:covbases}, where $U$ is the basis matrix of the column space of $\mathcal{A}_k(A\mathbf{z})$. Though the unseen bases in $A^\bot$ are not used directly, $\mu_A$ still has to depend on $A^\bot$. Whenever $A=\Id_m$, the generalized convolution coherence coincides with the convolution coherence. That is, $\mu_{\Id_m}(\mathbf{z})=\hat{\mu}_1(\mathbf{z})=\mu_1(\mathcal{A}_k(\mathbf{z})), \forall\mathbf{z}$. To examine the ranges of $\mu_A(\mathbf{z})$, we temporarily turn to the following quantity that is slightly different from $\mu_A(\mathbf{z})$:
\begin{align*}
\tilde{\mu}_A(\mathbf{z}) = \frac{q}{r}\max_{i\in[q],j\in[k]}\|U^TT^{j-1}\tilde{A}\mathbf{e}_i\|_2^2,
\end{align*}
where $\tilde{A}=[A, A^\bot]\in\mathbb{R}^{q\times{}q}$. Then it is easy to see that $1\leq\tilde{\mu}_A(\mathbf{z})\leq{}q$, and we will show that the lower and upper bounds are both attainable. Consider the case where $q=m$ (so $\mu_A(\mathbf{z})=\tilde{\mu}_A(\mathbf{z})$) and $\mathbf{c}$ is a $q$-dimensional nonzero constant vector. Let the SVD of $\mathbf{c}$ be $U_{\mathbf{c}}\Sigma_{\mathbf{c}}V_{\mathbf{c}}^T$ with $U_{\mathbf{c}}\in\mathbb{R}^{q\times{}q}$. Let $U_F\in\mathbb{R}^{q\times{}q}$ and $V_F\in\mathbb{R}^{q\times{}q}$ be constructed as in~\eqref{eq:dft:ufvf}. In this way, the following can be verified analytically: If $\mathbf{z}=\mathbf{c}$ and $A=V_FU_F^TU_{\mathbf{z}}^T$ then $\mu_A(\mathbf{z})=\tilde{\mu}_A(\mathbf{z})=1$; if $A=V_FU_F^T$ and $\mathbf{z}=A^T\mathbf{c}$ then $\mu_A(\mathbf{z})=\tilde{\mu}_A(\mathbf{z})=q$.\footnote{Here, the conclusions follow from the fact that $\mathcal{A}_{k}(A\mathbf{z})$ is a constant matrix and the first column and row of $V_FU_F^T$ are constant vectors.} In the general cases where $q\geq{}m$, since $\mu_A(\mathbf{z})\leq\tilde{\mu}_A(\mathbf{z})$, it is not impossible for $\mu_A(\mathbf{z})$ to fall below 1; this would be nice if possible. Up to present, we tend to believe that $\mu_A(\mathbf{z})$ is a positive quantity bounded from above by $q$. Note that, for any $\mathbf{z}\neq0$, it is impossible for $\mu_A(\mathbf{z})$ to be $0$, as $U^TA=0$ gives $\mathbf{z}=0$.\\

\noindent\textbf{Convolutional Basis Coherence:} Consider the situation where we have already obtained some transformation $A=[\mathbf{a}_1,\cdots,\mathbf{a}_m]\in\mathbb{R}^{q\times{}m}$ such that $A\mathbf{y}$ is convolutionally low-rank. In this case, for the LbCNNM program~\eqref{eq:lbcnnm:exact} to be able to identify the target $\mathbf{y}$, it is indeed necessary that the bases $\{\mathbf{a}_i\}_{i=1}^m$ themselves are \emph{not} convolutionally low-rank. Otherwise, $\mathbf{y}$ may be unidentifiable. Without loss of generality, assume that $\mathbf{a}_{j_1}$ is convolutionally low-rank, i.e., $\mathcal{A}_k(\mathbf{a}_{j_1})$ is low-rank, and suppose that the $j_1$th entry of $\mathbf{y}$ is missing. Denote by $\mathbf{e}_{j_1}$ the $j_1$th standard basis of $\mathbb{R}^m$. Construct another vector $\mathbf{y}'=\mathbf{y}+b\mathbf{e}_{j_1}, \forall{b}\neq0$. Then we have $\mathcal{P}_{\Omega}(\mathbf{y}')=\mathcal{P}_{\Omega}(\mathbf{y})$ and
\begin{align*}
\mathrm{rank}(\mathcal{A}_k(A\mathbf{y}'))\leq\mathrm{rank}(\mathcal{A}_k(A\mathbf{y}))+\mathrm{rank}(\mathcal{A}_k(\mathbf{a}_{j_1})),
\end{align*}
which means $A\mathbf{y}'$ can also be convolutionally low-rank and thus $\mathbf{y}$ may not be recoverable. To prevent such unidentifiable cases, we define a quantity as follows:
\begin{defn}[Convolutional Basis Coherence]\label{defn:basiscoherence} Let $\mathbf{a}\in\mathbb{R}^q$ be a basis vector obeying $\|\mathbf{a}\|_2=1$. Then the convolutional basis coherence of $\mathbf{a}$, denoted by $\bar{\mu}(\mathbf{a})$, is defined as
\begin{align*}
\bar{\mu}(\mathbf{a}) = \sum_{j=1}^q|\mathbf{a}^TT^{j-1}\mathbf{a}|,
\end{align*}
where $T\in\mathbb{R}^{q\times{}q}$ is the permutation matrix used in Definition \ref{defn:geconvcoherence} and $|\cdot|$ denotes the absolute value of a number.
\end{defn}

The smallest value $\bar{\mu}(\mathbf{a})$ can be is 1, achieved if $\mathbf{a}$ is a standard basis in $\mathbb{R}^q$. The maximum of $\bar{\mu}(\mathbf{a})$ is $q$, attained when all the entries in $\mathbf{a}$ have an equal value of $1/\sqrt{q}$. Notice, that the convolution matrix of $\mathbf{a}$ is $\mathcal{A}_k(\mathbf{a})=[\mathbf{a},T\mathbf{a},\cdots,T^{k-1}\mathbf{a}]$. Then it may be easily seen that the cosine similarity between the $i$th and $j$th column vectors of $\mathcal{A}_k(\mathbf{a})$ is given by
\begin{align*}
|(T^{i-1}\mathbf{a})^T(T^{j-1}\mathbf{a})|=|\mathbf{a}^TT^{j-i}\mathbf{a}|, \forall{}i,j\in[k],\forall{}k\in[q],
\end{align*}
where it is worth noting that, due to the peculiarities of circular convolution,  $T^{-i}=T^{q-i}$, $\forall{}i\in[q]$. Hence, in fact, $\bar{\mu}(\mathbf{a})$ measures the angles between the column vectors of the convolution matrix of $\mathbf{a}$, and $\bar{\mu}(\mathbf{a})$ may be large if $\mathcal{A}_k(\mathbf{a})$ is strictly or close to be low-rank.

Now, consider the transform matrix $A$ that contains multiple bases. Following the above justifications, it is natural to define the following quantity called \emph{convolutional transformation coherence}:
\begin{align}\label{eq:convtranscoh}
\bar{\mu}(A) = \max_{i\in[m]}\bar{\mu}(\mathbf{a}_i),
\end{align}
where $\mathbf{a}_i$ is the $i$th column vector of $A\in\mathbb{R}^{q\times{}m}$. The above definition is made specific to the setup of \emph{arbitrary} sampling pattern in Problem~\ref{pb:seqc}. If one just wants to consider the standard setup of time series forecasting, then the coherence in~\eqref{eq:convtranscoh} can be replaced by the following:
\begin{align*}
\tilde{\mu}(A) = \max_{m-h+1\leq{}i\leq{}m}\bar{\mu}(\mathbf{a}_i),
\end{align*}
where $h$ is the forecast horizon. In general, $\bar{\mu}(A)$ (resp. $\tilde{\mu}(A)$) ranges from 1 to $q$, where the minimum can be attained by $A=[\Id_m;0]\in\mathbb{R}^{q\times{}m}$, and the maximum is achieved if a constant basis appears in $A$ (resp. the last $h$ columns of $A$).
\section{Analyses and Learning Methods}\label{sec:meth_result}
In this section, we shall figure out under which conditions LbCNNM can succeed in recovering the target $\mathbf{y}$ at first, then building suitable learning methods with the attempt to meet those conditions.
\subsection{Recovery Conditions of LbCNNM}
The following theorem establishes sufficient conditions for the LbCNNM program~\eqref{eq:lbcnnm:exact} to exactly recover the target $\mathbf{y}$ from its observed part $\mathcal{P}_{\Omega}(\mathbf{y})$.
\begin{theo}[Noiseless]\label{main:thm:noiseless}
Denote by $r$ the rank of the convolution matrix $\mathcal{A}_k(A\mathbf{y})$, denote by $\hat{\mu}_2(A\mathbf{y})$ the second convolution coherence defined as in~\eqref{eq:convcoherence}, denote by $\mu_A(\mathbf{y})$ the first generalized convolution coherence of $\mathbf{y}$ defined in Definition~\ref{defn:geconvcoherence}. Denote by $\bar{\mu}(A)$ the convolutional transformation coherence of $A$ defined in \eqref{eq:convtranscoh}. Let $\mu=\max(\hat{\mu}_2(A\mathbf{y})\bar{\mu}(A),\mu_A(\mathbf{y}))$. Then $\mathbf{x}=\mathbf{y}$ is the unique minimizer to the problem in~\eqref{eq:lbcnnm:exact}, as long as
\begin{align*}
\rho_0=\frac{\mathrm{card}(\Omega)}{m} > 1 - \frac{0.25k}{\mu{}mr},
\end{align*}
where $k=\beta{q}$ ($0<\beta\leq1$) is the kernel size adopted by LbCNNM.
\end{theo}

When $A=\Id_m$, i.e., CNNM, the above theorem falls back to Theorem 3.1 of~\cite{liu:arxiv:2019}. To assess the optimality of the proven \emph{sampling bound}, the lower bound of the sampling complexity $\rho_0$, we consider a special category of data where the target $\mathbf{y}$ takes value 1 at the last $\tilde{h}$ ($1\leq\tilde{h}\leq{}m/2$) entries and 0 everywhere else; namely, $y_i=0$ for $1\leq{}i\leq{}m-\tilde{h}$ and $y_i=1$ for $m-\tilde{h} + 1\leq{}i\leq{}m$. For such data, whenever those $\tilde{h}$ ones are all missing, LbCNNM cannot, probably no method can, succeed in identifying $\mathbf{y}$. In other words, the optimal sampling bound is
\begin{align*}
\rho_0>1 - \tilde{h}/{m}.
\end{align*}
Set $k=q$ for simplicity. Let the SVD of $\mathbf{y}$ be $U_{\mathbf{y}}\Sigma_{\mathbf{y}}V_{\mathbf{y}}^T$, where $U_{\mathbf{y}}\in\mathbb{R}^{m\times{}m}$ is orthogonal. Construct the transform matrix as $A=V_FU_F^TB\in\mathbb{R}^{q\times{}m}$, where $B\in\mathbb{R}^{q\times{}m}$ is an orthonormal matrix obtained by padding $q-m$ zero rows to $U_{\mathbf{y}}^T$, i.e., $B=[U_{\mathbf{y}}^T;0]$, and $V_F\in\mathbb{R}^{q\times{}q}$ and $U_F\in\mathbb{R}^{q\times{}q}$ are two orthogonal matrices computed according to~\eqref{eq:dft:ufvf}. With such configurations, it can be verified that $r=1$, $\hat{\mu}_2(A\mathbf{y})=1$, $\mu_A(\mathbf{y}) = q/\tilde{h}$ and $\bar{\mu}(A)\approx2q/\tilde{h}$. Thus, roughly speaking, Theorem~\ref{main:thm:noiseless} proposes a bound as follows:\footnote{In this example, the SVD of $\mathbf{y}$ is not unique, and the value of $\bar{\mu}(A)$ essentially depends on how $U_{\mathbf{y}}$ is constructed. While using $U_{\mathbf{y}}=\left[\begin{array}{cc}
0 &\bar{V}_F\bar{U}_F^T\\
\tilde{V}_F\tilde{U}_F^T&0
\end{array}\right]$ with $\bar{V}_F,\bar{U}_F^T\in\mathbb{R}^{(m-\tilde{h})\times(m-\tilde{h})}$ and $\tilde{V}_F,\tilde{U}_F^T\in\mathbb{R}^{\tilde{h}\times{}\tilde{h}}$ being orthogonal matrices constructed according to~\eqref{eq:dft:ufvf}, we have experimentally verified that $\bar{\mu}(A)\approx2q/\tilde{h}$.}
\begin{align*}
\rho_0>1 - 0.125\tilde{h}/{m},
\end{align*}
which is nearly the same with the optimal bound. As a consequence, the sampling bound provided by Theorem~\ref{main:thm:noiseless} is indeed ``optimal'', in a sense that it is almost impossible to further reduce the bound on all kinds of data \emph{simultaneously}. Of course, it is entirely possible to obtain better bounds while taking into account only certain type of data.

In practice, the observed part is often a noisy version of $\mathcal{P}_{\Omega}(\mathbf{y})$---or the convolution matrix $\mathcal{A}_k(A\mathbf{y})$ is not strictly low-rank as equal. In this case, one should relax the equality constraint in LbCNNM, reaching the following convex program:
\begin{align}\label{eq:lbcnnm:noisy}
&\min_{\mathbf{x}\in\mathbb{R}^m} \norm{\mathcal{A}_k(A\mathbf{x})}_{*},\quad\textrm{s.t.}\quad\|\mathcal{P}_{\Omega}(\mathbf{x}-\hat{\mathbf{y}})\|_2\leq{}\epsilon,
\end{align}
where $\mathcal{P}_{\Omega}(\hat{\mathbf{y}})$ denotes an observation of $\mathcal{P}_{\Omega}(\mathbf{y})$, and $\epsilon\geq0$ is a parameter.

\begin{theo}[Noisy]\label{main:thm:noisy}
Use the notations defined in Theorem~\ref{main:thm:noiseless}. Let $\mathcal{P}_{\Omega}(\hat{\mathbf{y}})$ be an observed version of $\mathcal{P}_{\Omega}(\mathbf{y})$ that obeys $\|\mathcal{P}_{\Omega}(\hat{\mathbf{y}}-\mathbf{y})\|_2\leq{}\epsilon$. If
\begin{align*}
\rho_0=\frac{\mathrm{card}(\Omega)}{m} > 1 - \frac{0.22k}{\mu{}mr},
\end{align*}
then any optimal solution $\mathbf{x}$ to the LbCNNM program~\eqref{eq:lbcnnm:noisy} satisfies
\begin{align*}
\|\mathbf{x}-\mathbf{y}\|_2\leq{}(1+\sqrt{2})(38\sqrt{k}+2)\epsilon.
\end{align*}
\end{theo}

The above theorem shows that the upper bound of the recovery error scales with $\sqrt{k}$. Similar phenomena appear in many papers, e.g.,~\cite{CandesPIEEE}. This is indeed purely the cause of applying the general inequalities $\|\mathbf{z}\|_2\leq\|\mathbf{z}\|_1\leq\sqrt{k}\|\mathbf{z}\|_2, \forall{}\mathbf{z}\in\mathbb{R}^k$, but is unlikely to be optimal because the two inequalities are both used in the proof. Among the other things, since no additional hypothesis is made in our proofs, Theorem~\ref{main:thm:noiseless} and Theorem~\ref{main:thm:noisy} are applicable to any real-valued data vectors. This, however, does not say that LbCNNM can work well in any cases. In fact, the above theorems illustrate that the key for LbCNNM to succeed is to learn proper transformation $A$ to meet the success conditions.
\subsection{Learning Methods}
Regarding the question of how to learn the transform matrix $A\in\mathbb{R}^{q\times{}m}$ that leads to convolutional low-rankness, we will first show that the task boils down to the problem of learning another orthonormal matrix $B\in\mathbb{R}^{q\times{}m}$ that produces sparsity, and then establish proper learning methods to finish the job.
\subsubsection{Pipeline for Learning the Transform Matrix $A$}\label{sec:learnA:pipe}
The fact that the convolution matrix is a truncation of the circular matrix simply leads to $\rank{\mathcal{A}_k(A\mathbf{z})}\leq\|\mathcal{F}(A\mathbf{z})\|_0$, $\forall\mathbf{z}\in\mathbb{R}^m$. As a result, the transform matrix $A$ in LbCNNM can be obtained as follows:
\begin{itemize}
\item[1)] Given a data matrix $Y=[\mathbf{y}_1,\cdots,\mathbf{y}_n]\in\mathbb{R}^{m\times{}n}$ containing $n$ training samples, which is extracted from some training sequence as in Section~\ref{sec:gematrix:pcp} and Section~\ref{sec:argumentation}, learn an orthonormal matrix $B\in\mathbb{R}^{q\times{}m}$ such that $BY$ is sparse.
\item[2)] Compute two orthogonal matrices, $U_F\in\mathbb{R}^{q\times{}q}$ and $V_F\in\mathbb{R}^{q\times{}q}$, by using $w=q$ as the input for~\eqref{eq:dft:ufvf}. Then construct $A=V_FU_F^TB$.
\end{itemize}
Let $A$ be computed as above, and suppose that the learning of $B$ is successful such that $\|B\mathbf{y}_i\|_0\leq{}s$, $\forall{}i\in[n]$. Then it follows from the properties of the Fourier transform that (see Section~\ref{sec:dft}):
\begin{align*}
\rank{\mathcal{A}_k(A\mathbf{y}_i)}\leq\|\mathcal{F}(A\mathbf{y}_i)\|_0\leq2s, \forall{}i\in[n].
\end{align*}

Among the other things, the learnt transform matrix $A$ actually owns out-of-sample generalization ability, in a sense that
\begin{align}\label{eq:generalization}
\rank{\mathcal{A}_k(A\mathbf{z})}\leq2ts, \forall{}\mathbf{z}\in\mathcal{S}_t(Y),
\end{align}
where $\mathcal{S}_t(Y)$, called \emph{generalizable space} of $Y$, is the space of vectors that can be represented as a linear combination of $t'\leq{}t$ column vectors of $Y$.\footnote{The generalizable space $\mathcal{S}_t(Y)$ is a nonlinear space that can include as a subset the union of the local subspaces spanned by $t'\leq{}t$ adjacent samples in $Y$.} So $t=1$ for the training samples, and thus it is desirable to use a large number of training samples to learn the transform matrix $B$. Moreover, the learnt model could be robust against noise: Whenever $B\mathbf{y}_i$ is sparse only in an approximate sense, the corresponding convolution matrix $\mathcal{A}_k(A\mathbf{y}_i)$ will be approximately low-rank---in this case Theorem~\ref{main:thm:noisy} will show its value.

Without imposing any restrictive assumptions, however, there is no way to ensure that the target $\mathbf{y}$, which is essentially unknown, can be contained by the generalizable space, $\mathcal{S}_t(Y)$, with a sufficiently small $t$. This is the ``notorious'' \emph{generalization issue}, which is probably the most difficult part of time series forecasting, in which the phenomenon of \emph{covariate shift}~\cite{sug:jmlr:2007} is indeed ubiquitous. In the rest of this subsection, we shall focus on how to learn the desired transformation $B$ from a given data matrix $Y$. As for the generalization issue, we will introduce some heuristics to improve the generalization performance of LbCNNM in Section~\ref{sec:argumentation}.
\subsubsection{Coherence Issue}
Given a data matrix $Y\in\mathbb{R}^{m\times{}n}$, now it is clear that the majority of learning the transformation $A$ is to find another orthonormal matrix $B\in\mathbb{R}^{q\times{}m}$ such that $BY$ is sparse. This is closely related to the well-known problem of sparsely-used dictionary learning. But there is an important difference: Whenever one has found $B$ such that $E=BY$ is sparse, then $B^T$ is surely a sparsely-used dictionary such that $Y=B^TE$ and $E$ is sparse. However, the converse is unnecessarily true, as $Y=B^TE$ with $E$ being sparse cannot guarantee that $BY$ is sparse---in fact $BY$ is often dense.
\begin{figure}[h!]
\begin{center}
\includegraphics[width=0.4\textwidth]{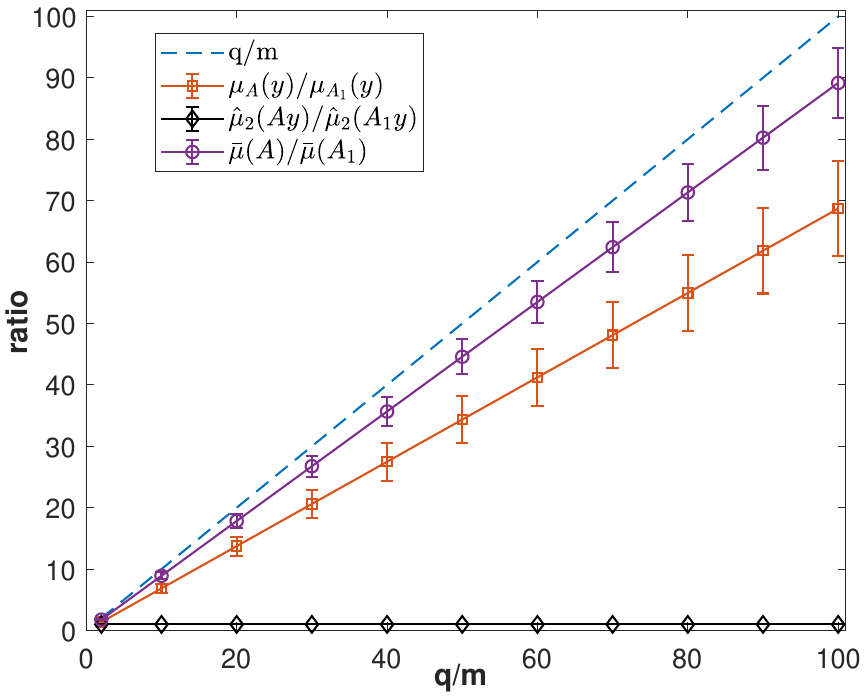}\vspace{-0.15in}
\caption{Investigating the behavior of coherence when expanding the size of $B$ by padding with zeros. We set $m=10$, $q=cm$ with $c=\{2,10,20,\cdots,100\}$, and $k=q$. The experimental data is created as follows: $B_1\in\mathbb{R}^{m\times{}m}$ is an orthogonal matrix generated at random, $\mathbf{y}=B_1^T\mathbf{z}$ where $\mathbf{z}\in\mathbb{R}^{m}$ is a randomly generated sparse vector with $\|\mathbf{z}\|_0=2$, $A_1=V_FU_F^TB_1\in\mathbb{R}^{m\times{}m}$ and $A\in\mathbb{R}^{q\times{}m}$ is formed from $[B_1;0]\in\mathbb{R}^{q\times{}m}$ in a similar way. The numbers plotted in the above figure are the means and standard deviations collected from 100 trials.}\label{fig:ratio}
\end{center}
\end{figure}

According to Theorem~\ref{main:thm:noiseless} and Theorem~\ref{main:thm:noisy}, the sampling complexity required by LbCNNM is $1 - \mathcal{O}(q/(\mu{}mr))$ (note that $k=\mathcal{O}(q)$). Thus, it seems that a ``satisfactory'' transform matrix $B$ can be obtained effortlessly. That is, first obtain somehow, for example, generate randomly, an orthogonal matrix $B_1\in\mathbb{R}^{m\times{}m}$, and then expand the size of $B_1$ to $q\times{}m$ via padding with zeros; namely, $B=[B_1;0]\in\mathbb{R}^{q\times{}m}$. This seems ``nice'' from the view point of sparsity, because one can always put $q\gg{}m$ and thus $BY$ looks sparse, in a sense that $\|B\mathbf{y}_i\|_0\leq{}m\ll{}q,\forall{}i$. Unfortunately, the convolutional transformation coherence $\bar{\mu}(A)$ may increase as $q$ goes large. This is because, whenever $q\gg{}m$, the bases in $A=V_FU_F^TB$ will be convolutionally low-rank:
\begin{align*}
\mathrm{rank}(\mathcal{A}_k(A\mathbf{e}_i))\leq2\|B\mathbf{e}_i\|_0\leq2m,
\end{align*}
where $\mathbf{e}_i$ is the $i$th standard basis of $\mathbb{R}^m$.

In addition, the generalized convolution coherence, $\mu_A$, may also increase when one adds many zero rows to $B$. To confirm, let's consider the ratio of $\mu_A$ to $\mu_{A_1}$, where $\mu_{A_1}$ is the generalized convolution coherence given by the transform matrix $A_1\in\mathbb{R}^{m\times{}m}$ formed from $B_1\in\mathbb{R}^{m\times{}m}$. For the example (in Section~\ref{sec:conv:coherence}) used for showing that $\mu_{A_1}$ can be 1, it is simple to prove that $\mu_{A}/\mu_{A_1}=q/m$. However, in general cases, the relation between $\mu_{A}/\mu_{A_1}$ and $q/m$ is much more complicated---and essentially data-dependent. Within the scope of randomly generated data, our simulation results shown in Figure~\ref{fig:ratio} suggest that $\mu_{A}/\mu_{A_1}$ may be a random variable whose mean and standard deviation are directly proportional to $q/m$, and similarly for $\bar{\mu}(A)/\bar{\mu}(A_1)$. Whenever either $\bar{\mu}(A)/\bar{\mu}(A_1)$ or $\mu_A/\mu_{A_1}$ reaches $\mathcal{O}(q/m)$, the sampling complexity required by LbCNNM becomes $1-\mathcal{O}(1/r)$, which means there are no benefits to simply add some zero rows to $B$.

Still, highly-overcomplete dictionaries (i.e., $q\gg{}m$) have good potential. But, to demonstrate their power, one needs to find a way to overcome the coherence issue, which is rather challenging. In this work, we shall focus on learning slightly-overcomplete dictionaries with $q=2m$ by Principal Component Analysis (PCA) and PCP.
\begin{algorithm}[htb]
\caption{Learning the Transform Matrix $A$ by PCA}
\label{alg:learnA:pca}
\begin{algorithmic}[1]
\STATE \textbf{input}: a data matrix $Y\in\mathbb{R}^{m\times{}n}$.
\STATE \textbf{output}: an orthonormal matrix $A\in\mathbb{R}^{2m\times{}m}$.
\STATE compute the SVD of $Y$ as $Y=U_Y\Sigma_YV_Y^T$, with $U_Y\in\mathbb{R}^{m\times{}m}$ being orthogonal.
\STATE set $B=[U_Y^T;0]\in\mathbb{R}^{2m\times{}m}$.
\STATE compute two orthogonal matrices, $U_F\in\mathbb{R}^{2m\times{}2m}$ and $V_F\in\mathbb{R}^{2m\times{}2m}$, as in~\eqref{eq:dft:ufvf}.
\STATE return $A=V_FU_F^TB$.
\end{algorithmic}
\end{algorithm}
\subsubsection{Learning the Transform Matrix $B$ Based on PCA}\label{sec:pca}
Whenever the given data matrix $Y$ is fairly low-rank,\footnote{Within the scope of univariate series, as we have explained in Section~\ref{sec:gematrix:pcp}, the trends and seasonality arguably result in low-rank generation matrices. Note here that, the presence of trends and seasonality is a sufficient but not necessary condition for low-rankness, as it is entirely possible for something else to induce low-rankness.} the required transform matrix $B$ can be found by PCA in an efficient way: $B = [U_Y^T;0]\in\mathbb{R}^{2m\times{}m}$ with $U_Y\in\mathbb{R}^{m\times{}m}$ being an orthogonal matrix consisting of all $m$ left singular vectors of $Y$. Here, the operation of padding with zeros makes only very mild difference in forecasting accuracy and is mainly for the sake of consistence.\footnote{Empirically, we did observe that $B =[U_Y^T; 0]\in\mathbb{R}^{2m\times{}m}$ gains very mild improvement over $B=U_Y^T\in\mathbb{R}^{m\times{}m}$. The reason is probably because we set $k=0.5q$, under which $k$ will increase from $0.5m$ to $m$ when $q$ is changed from $m$ to $2m$.} The computational procedure is given in Algorithm~\ref{alg:learnA:pca}. While quite simple, this algorithm is provably successful provided that $Y$ is strictly or very close to be low-rank. More precisely, let $\rank{Y}=r$ and $A$ be learnt as in Algorithm~\ref{alg:learnA:pca}, then all the samples in $Y$ will be transformed to have a convolution rank of at most $2r$:
\begin{align*}
\rank{\mathcal{A}_k(A\mathbf{y}_i)}\leq2r, \forall{}i\in[n].
\end{align*}
Consider a special case where $Y$ is a vector, i.e., an $m\times{}1$ matrix. Then the transformation $A$ learnt by Algorithm~\ref{alg:learnA:pca} guarantees that $\rank{\mathcal{A}_k(AY)}\leq2$ (in fact $\rank{\mathcal{A}_k(AY)}=1$ because $\mathcal{A}_k(AY)$ is a constant matrix), proving that any data vector can be made compliant to the convolutional low-rankness condition. In the presence of noise, namely $Y = \bar{Y}+N$ with $\|N\|_F\leq\epsilon$ and $\bar{Y}=[\bar{\mathbf{y}}_1,\cdots,\bar{\mathbf{y}}_n]$ being the best rank-$r$ approximation of $Y$, Algorithm~\ref{alg:learnA:pca} guarantees that the transformed signal $A\mathbf{y}_i$ is convolutionally low-rank in an approximate sense:
\begin{align*}
 &\sum_{i=1}^n\|\mathcal{A}_k(A\mathbf{y}_i) - \mathcal{A}_k(A\bar{\mathbf{y}}_i)\|_F^2\leq{}k^2\epsilon^2 \\
 &\textrm{ with } \rank{\mathcal{A}_k(A\bar{\mathbf{y}}_i)}\leq2r, \forall{}i\in[n],
\end{align*}
which suggests that Algorithm~\ref{alg:learnA:pca} may work well as long as $r$ and $\epsilon$ are sufficiently small, i.e., the given $Y$ is very close to be low-rank.
\begin{figure*}
\begin{center}
\includegraphics[width=0.95\textwidth]{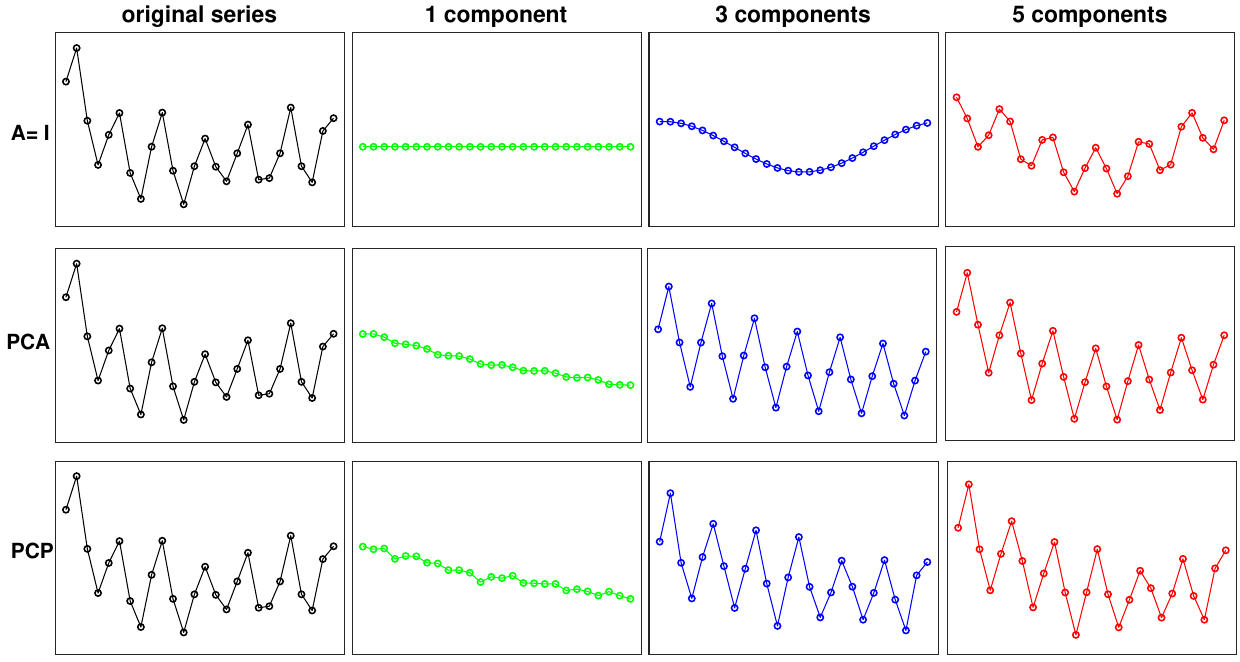}\vspace{-0.15in}
\caption{Investigating the effects of learning the transform matrix $A$. From left to right: a time series $\mathbf{z}$ with dimension $m=26$, the signal reconstructed from 1 principal component of the convolution matrix of $A\mathbf{z}$, the reconstruction from 3 principal components, and the reconstruction given by 5 principal components. In these experiments, we adopt $Y=G_0$ with $G_0$ being the generation matrix extracted from a training sequence of length 59.}\label{fig:learntA}
\end{center}
\end{figure*}

Since $A$ is an orthonormal matrix, for any series $\mathbf{z}\in\mathbb{R}^m$, its transformed signal $A\mathbf{z}$ preserves everything in $\mathbf{z}$. For this reason, Algorithm~\ref{alg:learnA:pca} is different from the other sparse representation learning methods such as Sparse Filtering~\cite{sf:nips:2011}, which aims at extracting sparse features from data in an \emph{irreversible} fashion. In our method, by contrast, the purpose of learning $A$ is not about feature exaction, but instead to reorganize the data such that the structures of interest are captured by the principal components of the convolution matrix of $A\mathbf{z}$. According to Theorem~\ref{main:thm:noisy}, the principal components dominate the forecasting results and the tail components may be treated as noise.

To see what is really encoded in the principal components (of the convolution matrix of the transformed signal), one just needs to examine the signal reconstructed as $\hat{\mathbf{z}}=A^T\mathcal{F}^{-1}(\mathcal{P}_{[r]}(\mathcal{F}(A\mathbf{z})))$, where $\mathcal{F}(\cdot)$ is the DFT operator, and $\mathcal{P}_{[r]}(\cdot)$ is an operator that preserves only $r$ largest values---in terms of magnitudes---of a vector and zeros out the others. Figure~\ref{fig:learntA} shows an example. Without representation learning, i.e., $A=\Id$, the principal components fail to reconstruct the original series, which implies that LbCNNM with $A=\Id$ (i.e., CNNM) will produce poor forecasts. By contrast, after learning a transformation by PCA, five principal components can approximately reconstruct the original series (see the second row of Figure~\ref{fig:learntA}), exhibiting the value of representation learning. However, the reconstruction is not good enough, which means the produced forecasting results will be imperfect. The reason is because the series shown in Figure~\ref{fig:learntA} contains a considerable amount of dynamics, which can ``crack'' the low-rank structures of the generation matrix $G_0$. Hence, PCA is indeed not up to the mark, and a better approach is to recover the intrinsic low-rank structure underlying $Y$ by PCP, as will be shown in the next subsection.
\subsubsection{Learning the Transform Matrix $B$ Based on PCP}\label{sec:pcp}
Whenever the observations are grossly corrupted such that $Y$ is far from being low-rank, the PCA-based Algorithm~\ref{alg:learnA:pca} may fail in pursuing convolutional low-rankness. Under the context of time series forecasting, this may happen in the presence of dynamics. In other words, while using $Y=G_0$, the low-rankness arising from seasonality and trends can be concealed by the dynamics. Fortunately, the dynamics are often rare and happen only on few observations---otherwise the problem would be insolvable, i.e., the matrix that stores dynamics is often sparse. As a consequence, it is feasible to recover the intrinsic low-rank structure underlying $Y$ by PCP~\cite{Candes:2009:JournalACM}:
\begin{align}\label{eq:pcp}
\min_{L,S}\|L\|_* + \lambda_{\mathrm{pcp}}\|S\|_1, \quad\textrm{s.t.}\quad{}Y = L + S,
\end{align}
where $\lambda_{\mathrm{pcp}}$ is a parameter. In this work, we consistently set $\lambda_{\mathrm{pcp}}=1/\sqrt{\max(m, n)}$ according to the analysis in~\cite{Candes:2009:JournalACM}.

Figure~\ref{fig:pcpts} shows an example, which illustrates that the low-rank matrix $L$ corresponds to the seasonality and trends and the sparse term $S$ is mostly consist of dynamics. Note here that the low-rank component $L$ is a mixture of seasonality and trends---and possibly includes something else as well, thus PCP is different from \emph{trend filtering}~\cite{Kin:2020:siam,Bruder11trendfiltering} and \emph{empirical mode decomposition}~\cite{huang:1998:emd}. Just for information, there are many source separation tools in the literature that did a similar job with PCP, e.g., Low-Rank Representation (LRR)~\cite{liu:tpami:2016b}, Latent LRR (LatLRR)~\cite{liu:2011:iccv}, etc. Empirically, we have found that LbCNNM works almost equally while using PCP, LRR or LatLRR (the difference is smaller than 0.05\%), thereby we suggest PCP due to its transparency in choosing the parameter $\lambda_{\mathrm{pcp}}$.
\begin{figure}[h!]
\begin{center}
\includegraphics[width=0.48\textwidth]{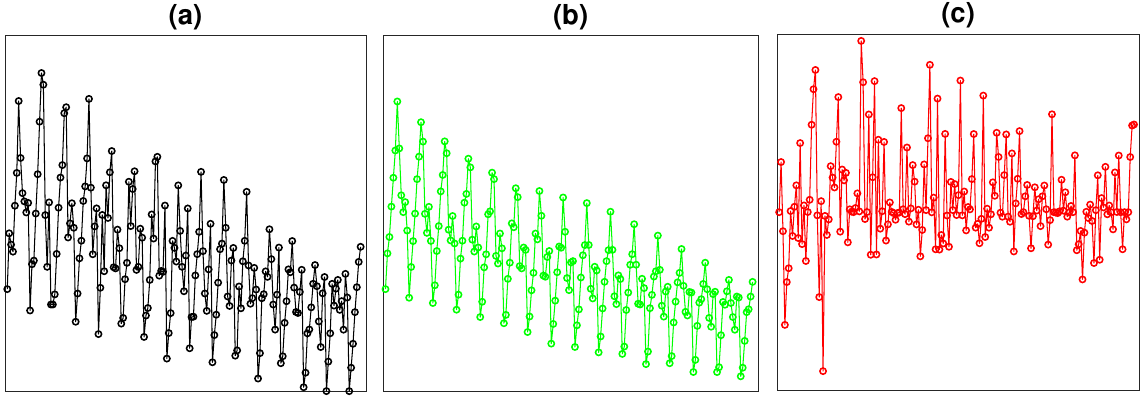}\vspace{-0.15in}
\caption{Investigating the effects of using PCP to decompose the generation matrix $G_0$ into a low-rank component $L$ plus a sparse component $S$. (a) A time series of length 187. (b) The signal reconstructed from the low-rank component $L$. (c) The signal reconstructed from $S$.}\label{fig:pcpts}
\end{center}
\end{figure}

Given that one needs only the seasonality and trends, the sparse matrix $S$ should be discarded. This is seemingly reasonable, as the dynamics are so elusive as to be barely operable. However, for LbCNNM to achieve high forecasting accuracy, it is desirable to find a way to convert the dynamics into something as convolutionally low-rank as possible. This is not ridiculous, as the dynamics, which are usually recognized as the gross components outside the scope of trends and seasonality, could have trackable patterns. Here, we would suggest a method that can handle the dynamics in an elegant way. First, we compute the SVD of $L$ as $L=U_L\Sigma_LV_L^T$, with $U_L\in\mathbb{R}^{m\times{}m}$ being orthogonal. Second, we construct a sparse matrix $E=[U_L^TL;S]\in\mathbb{R}^{2m\times{}n}$ by concatenating $U_L^TL$ and $S$ together vertically. Finally, we solve for $B$ the following $\ell_1$ minimization problem:
\begin{align}\label{eq:learnB}
\min_{B\in\mathbb{R}^{2m\times{}m}}\|BY-E\|_1,\quad\mathrm{s.t.}\quad{}B^TB=\Id_m,
\end{align}
which is a non-convex problem yet can be solved by Alternating Direction Method of Multipliers (ADMM)~\cite{admm:1976,alm:2009:lin}. The ADMM algorithm we use is standard, which minimizes the augmented Lagrangian function,
\begin{align*}
&\norm{Z}_{1}+\langle{}BY-E-Z,W\rangle+\frac{\rho}{2}\|BY-E-Z\|_F^2,
\end{align*}
with respect to $B$ and an auxiliary variable $Z\in\mathbb{R}^{2m\times{}m}$ alternately, and then updates the Lagrange multiplier $W\in\mathbb{R}^{2m\times{}m}$ as well as the penalty parameter $\rho>0$. Hereafter, $\langle\cdot,\cdot\rangle$ denotes the inner product between matrices or vectors. As acknowledged by lots of reports, \emph{ADMM always converges fast to the same global minimizer regardless of the initialization}~\cite{sun:2016:tit}. Unfortunately, while dealing with the non-convex optimization problems such as~\eqref{eq:learnB}, there is no theoretical guarantee for ADMM to converge to the critical points, not even the global minimizers. A good news is that it is not at all troublesome even if ADMM converges to a suboptimal solution, as what LbCNNM really cares is whether $BY$ is sparse enough, no matter whether the solution is globally optimal. A concrete evidence is that, if replacing the $\ell_1$ loss by $\ell_2$, the problem in~\eqref{eq:learnB} has a closed-form solution, which is exactly the solution found by ADMM at the first iteration. Since the sparsity of $BY$ mainly comes from the sparse matrix $E$ instead of the $\ell_1$ loss, $\ell_2$ also works well---we have empirically found that $\ell_1$ loss yields an improvement rate of only $0.7\%$ over $\ell_2$.
\begin{algorithm}[htb]
\caption{Learning the Transform Matrix $A$ by PCP}
\label{alg:learnA}
\begin{algorithmic}[1]
\STATE \textbf{input}:  a data matrix $Y\in\mathbb{R}^{m\times{}n}$.
\STATE \textbf{output}: an orthonormal matrix $A\in\mathbb{R}^{2m\times{}m}$.
\STATE decompose $Y$ into a low-rank term $L$ and a sparse term $S$ by PCP.
\STATE compute the SVD of $L$ as $L=U_L\Sigma_LV_L^T$ with $U\in\mathbb{R}^{m\times{}m}$ being orthogonal, and construct $E=[U_L^TL;S]$.
\STATE obtain $B\in\mathbb{R}^{2m\times{}m}$ via solving the problem in~\eqref{eq:learnB} by ADMM.
\STATE compute two orthogonal matrices, $U_F\in\mathbb{R}^{2m\times{}2m}$ and $V_F\in\mathbb{R}^{2m\times{}2m}$, as in~\eqref{eq:dft:ufvf}.
\STATE return $A=V_FU_F^TB$.
\end{algorithmic}
\end{algorithm}

Algorithm~\ref{alg:learnA} summarizes the whole procedure of the PCP-based learning algorithm. Indeed, Algorithm~\ref{alg:learnA:pca} is a special case of Algorithm~\ref{alg:learnA}, as PCP falls back to PCA whenever $\lambda_{\mathrm{pcp}}=+\infty$. As shown in Figure~\ref{fig:learntA}, Algorithm~\ref{alg:learnA} is better than Algorithm~\ref{alg:learnA:pca} in restoring the original series from few principal components, inferring that Algorithm~\ref{alg:learnA} may outperform Algorithm~\ref{alg:learnA:pca} in terms of forecasting accuracy. Interestingly, as can be seen from the bottom row of Figure~\ref{fig:learntA}, there are dynamics in the signal reconstructed from 5 or even 3 principal components (of the convolution matrix of the transformed signal). That is, certain kind of dynamics, which somehow own trackable patterns, are successfully transformed to be convolutionally low-rank and are therefore predictable. Of course, it might be impossible to represent all kinds of dynamics in a convolutionally low-rank fashion, as convolutional low-rankness implies exact recovery---it isn't so realistic to exactly predict all kinds of dynamics possibly appear in future.
\begin{figure}[h!]
\begin{center}
\includegraphics[width=0.48\textwidth]{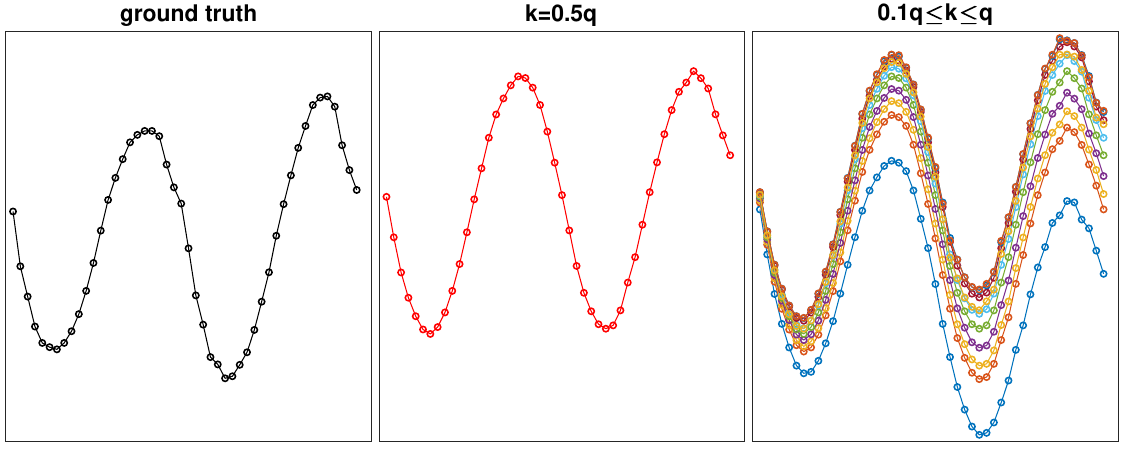}\vspace{-0.15in}
\caption{Illustrating the role of the kernel size $k$. Left: the $h=48$ future values we wish to predict. Middle: an estimate produced by LbCNNM with $k=0.5q$. Right: ten estimates produced by $k=0.1q,0.2q,\cdots,q$. In this experiment, the transform matrix $A$ is learnt by Algorithm~\ref{alg:learnA}, using $Y=G_0$. The length of the training sequence is 700, and the model size is $m=168$.}\label{fig:interval}
\end{center}
\end{figure}
\subsection{Discussions}\label{sec:discussion}
\noindent\textbf{On Interval Forecasting.} While the model in~\eqref{eq:lbcnnm:exact} is seemingly specific to point forecasting, LbCNNM actually owns the ability of providing intervals for its forecasts. The secret of success here is about the kernel size $k$. Under the background of point forecasting, we have experimentally found that $k=0.5q$ is a good choice---in fact $k=0.5q$ is near-optimal. Nevertheless, there is no unique way to determine this parameter, and the uncertainty in choosing $k$ unavoidably leads to the uncertainty of the produced forecasts. This is indeed expected, as illustrated in Figure~\ref{fig:interval}. As one can see, the forecasts produced by LbCNNM with $k=0.5q$ have considerable difference with the true values, but the intervals formed from the forecasts by multiple kernel sizes can cover the ground truth completely. Yet, to obtain tight, accurate intervals from the forecasts made by LbCNNM with multiple kernel sizes, some well-designed probabilistic models are necessary, deserving further studies.\\

\noindent\textbf{On Tensor-Valued Series.} Since high-order tensors can always be reshaped into vectors, LbCNNM has already owned the ability of dealing with tensor-valued time series. However, due to the issue of computational efficiency, we would not recommend the current LbCNNM for high-order tensors. To process an order-$b$ tensor with dimension $m_1\times\cdots{}\times{}m_b$, the model size is as large as $m=\Pi_{i=1}^bm_i$, which is often huge and may raise unaffordable computation budgets. Besides the computational issue, it is inadvisable either to treat the non-time dimensions of tensors in the same way as the time dimension. In general, while coping with high-order tensors, one should take the multi-dimensional correlations into full consideration, rather than simply vectorizing the tensors. How to extend LbCNNM to high-order tensors is also worthy of further investigations.
\section{Mathematical Proofs}\label{sec:proofs}
The following two lemmas are adapted from~\cite{liu:tpami:2019}. They are more general than their original versions in~\cite{liu:tpami:2019}, but the proofs are almost identical. Thus we omit the proof details for avoiding duplication.
\begin{lemm}[Lemma 5.6~\cite{liu:tpami:2019}]\label{lem:basic:quivalence}
Let $\mathcal{P}_1$ and $\mathcal{P}_2$ be two orthogonal projections onto some subspaces of $\Re^{a\times{}b}$, $\forall{a,b}$. Then the following are equivalent: 1) $\mathcal{P}_1\mathcal{P}_2\mathcal{P}_1$ is invertible, 2) $\|\mathcal{P}_1\mathcal{P}_2^\bot\mathcal{P}_1\|<1$, and 3) $\mathcal{P}_1\cap{}\mathcal{P}_2^\bot=\{0\}$.
\end{lemm}
\begin{lemm}[Lemma 5.12~\cite{liu:tpami:2019}]\label{lem:optnorm:invpt}
Let $\mathcal{P}_1$ and $\mathcal{P}_2$ be two orthogonal projections onto some subspaces of $\Re^{a\times{}b}$, $\forall{a,b}$. If $\mathcal{P}_1\mathcal{P}_2\mathcal{P}_1$ is invertible, then we have
\begin{align*}
\|\mathcal{P}_1^\bot\mathcal{P}_2\mathcal{P}_1(\mathcal{P}_1\mathcal{P}_2\mathcal{P}_1)^{-1}\| = \sqrt{\frac{1}{1-\|\mathcal{P}_1\mathcal{P}_2^\bot\mathcal{P}_1\|}-1}.
\end{align*}
\end{lemm}
\subsection{Proof of Theorem~\ref{main:thm:noiseless}}
The program in~\eqref{eq:lbcnnm:exact} is convex, and thus the proof roadmap is rather standard and very similar to the existing papers such as~\cite{Candes:2009:math,liu:arxiv:2019}. But there are still some technical difficulties need to overcome. In particular, the key to accomplish the proof is to construct a special operator as in the following:
\begin{align}
\mathcal{P}_{\Theta}(Z) = \sum_{i=1}^kT^{i-1}\tilde{A}D\tilde{A}^TT^{1-i}Z\mathbf{e}_i\mathbf{e}_i^T, \forall{Z}\in\mathbb{R}^{q\times{}k},
\end{align}
where $\tilde{A}$ and $T$ are defined in the same way as in~\eqref{eq:covbases}, $\mathbf{e}_i$ is the $i$th standard basis of $\mathbb{R}^k$, and $D\in\mathbb{R}^{q\times{}q}$ is a diagonal matrix constructed by expanding the sampling set $\Omega$. Namely, $D=\diag{\delta_1,\cdots,\delta_q}$, where $\delta_a = 1$ if either $a\in\Omega$ or $a>m$ and $\delta_a=0$ otherwise---the diagonal of $D$ has $(1-\rho_0)m$ zeros. It can be verified that $\mathcal{P}_{\Theta}$ is an orthogonal projection onto some subspace of $\mathbb{R}^{q\times{}k}$, and the orthogonal complement of $\mathcal{P}_{\Theta}$ is given by
\begin{align*}
&\mathcal{P}_{\Theta}^\bot(Z) = Z - \mathcal{P}_{\Theta}(Z) \\
&= \sum_{i=1}^kT^{i-1}\tilde{A}(\Id_q-D)\tilde{A}^TT^{1-i}Z\mathbf{e}_i\mathbf{e}_i^T, \forall{Z}\in\mathbb{R}^{q\times{}k}.
\end{align*}

Hereafter, denote by $(\cdot)^*$ the Hermitian adjoint (or conjugate) of a linear operator. The lemma below states some properties pretaining to $\mathcal{P}_{\Theta}$.
\begin{lemm}\label{lem:basic:adjoint}
For any $Z\in\mathbb{R}^{q\times{}k}$ and $\mathbf{z}\in\mathbb{R}^m$, we have
\begin{align*}
&A^T\mathcal{A}_k^*\mathcal{P}_{\Theta}(Z) = \mathcal{P}_{\Omega}A^T\mathcal{A}_k^*(Z)\textrm{ and }\mathcal{P}_{\Theta}\mathcal{A}_kA(\mathbf{z}) = \mathcal{A}_kA\mathcal{P}_\Omega(\mathbf{z}).
\end{align*}
\end{lemm}
\begin{proof}First of all, the definition of convolution matrix gives that
\begin{align*}
&\mathcal{A}_k(\mathbf{c}) = \sum_{i=1}^kT^{i-1}\mathbf{c}\mathbf{e}_i^T,\forall\mathbf{c}\in\mathbb{R}^q,\\\
&\mathcal{A}_k^*(C) = \sum_{i=1}^kT^{1-i}C\mathbf{e}_i,\forall{}C\in\mathbb{R}^{q\times{}k}.
\end{align*}
Regarding the first claim, we have
\begin{align*}
&A^T\mathcal{A}_k^*\mathcal{P}_{\Theta}(Z)= \sum_{i,j=1}^kA^TT^{1-j}T^{i-1}\tilde{A}D\tilde{A}^TT^{1-i}Z\mathbf{e}_i\mathbf{e}_i^T\mathbf{e}_j \\
&= \sum_{i=1}^kA^T\tilde{A}D\tilde{A}^TT^{1-i}Z\mathbf{e}_i =A^T\tilde{A}D\tilde{A}^T\sum_{i=1}^kT^{1-i}Z\mathbf{e}_i\\
&= A^T\tilde{A}D\tilde{A}^T\mathcal{A}_k^*(Z) = \mathcal{P}_{\Omega}A^T\mathcal{A}_k^*(Z).
\end{align*}
For the second claim, we have
\begin{align*}
&\mathcal{P}_{\Theta}\mathcal{A}_kA(\mathbf{z}) =  \sum_{i,j=1}^kT^{i-1}\tilde{A}D\tilde{A}^TT^{1-i}T^{j-1}A\mathbf{z}\mathbf{e}_j^T\mathbf{e}_i\mathbf{e}_i^T \\&=\sum_{i=1}^kT^{i-1}\tilde{A}D\tilde{A}^TA\mathbf{z}\mathbf{e}_i^T=\mathcal{A}_k\tilde{A}D\tilde{A}^TA(\mathbf{z})=\mathcal{A}_kA\mathcal{P}_\Omega(\mathbf{z}).
\end{align*}
\end{proof}
\subsubsection{Dual Condition}
The following lemma establishes the dual conditions under which the solution to the LbCNNM problem in~\eqref{eq:lbcnnm:exact} is unique and exact.
\begin{lemm}\label{lem:dual}Let the skinny SVD of $\mathcal{A}_k(A\mathbf{y})$ be $U\Sigma{}V^T$. Denote by $\mathcal{P}_{T}(\cdot)=\mathcal{P}_U(\cdot)+\mathcal{P}_V(\cdot)-\mathcal{P}_U\mathcal{P}_V(\cdot)$ the orthogonal projection onto the sum of the column space $\mathcal{P}_U$ and the row space $\mathcal{P}_V$. Then $\mathbf{y}$ is the unique minimizer to~\eqref{eq:lbcnnm:exact} provided that:
\begin{itemize}
\item[1.] $\mathcal{P}_{\Theta}^\bot\cap\mathcal{}P_{T}=\{0\}$.
\item[2.] There exists $H\in\mathbb{R}^{q\times{}k}$ such that $\mathcal{}P_{T}\mathcal{P}_{\Theta}(H) = UV^T$ and $\|\mathcal{}P_{T}^\bot\mathcal{P}_{\Theta}(H)\|<1$.
\end{itemize}
\end{lemm}
\begin{proof}
Take $W = \mathcal{}P_{T}^\bot\mathcal{P}_{\Theta}(H)$. Then $A^T\mathcal{A}_k^*(UV^T+W) = A^T\mathcal{A}_k^*\mathcal{P}_{\Theta}(H)$. By Lemma~\ref{lem:basic:adjoint},
\begin{align*}
A^T\mathcal{A}_k^*\mathcal{P}_{\Theta}(H) = \mathcal{P}_{\Omega}A^T\mathcal{A}_k^*(H)\in\mathcal{P}_{\Omega}.
\end{align*}
By the standard convexity arguments, $\mathbf{y}$ is an optimal solution to problem in~\eqref{eq:lbcnnm:exact}. So it remains to prove that $\mathbf{y}$ is the unique minimizer. To this end, consider a feasible solution $\mathbf{y}+\mathbf{a}$ such that $\mathcal{P}_{\Omega}(\mathbf{a}) = 0$. Then it follows from Lemma~\ref{lem:basic:adjoint} that
\begin{align*}
&\mathcal{P}_{\Theta}\mathcal{A}_kA(\mathbf{a}) = \mathcal{A}_kA\mathcal{P}_{\Omega}(\mathbf{a}) = 0,\\
&\langle{}\mathcal{P}_{\Theta}(H), \mathcal{A}_kA(\mathbf{a})\rangle =  \langle{}H, \mathcal{P}_{\Theta}\mathcal{A}_kA(\mathbf{a})\rangle = 0.
\end{align*}
Next, we shall show that the objective strictly grows unless $\mathbf{a}=0$. By the convexity of the nuclear norm,
\begin{align*}
&\|\mathcal{A}_kA(\mathbf{y}+\mathbf{a})\|_* - \|\mathcal{A}_kA(\mathbf{y})\|_* \geq \langle{}A^T\mathcal{A}_k^*(UV^T+F), \mathbf{a}\rangle\\
&= \langle{}UV^T+F, \mathcal{A}_kA(\mathbf{a})\rangle
\end{align*}
holds for any $F\in\mathcal{P}_{T}^\bot$ that obeys $\|F\|\leq{}1$. Choose $F$ such that $\langle{}F, \mathcal{A}_kA(\mathbf{a})\rangle = \|\mathcal{P}_T^\bot\mathcal{A}_kA(\mathbf{a})\|_*$. As a result, we have
\begin{align*}
&\langle{}UV^T+F, \mathcal{A}_kA(\mathbf{a})\rangle = \langle{}\mathcal{P}_{\Theta}(H)+F - W, \mathcal{A}_kA(\mathbf{a})\rangle\\
&= \langle{}F - W, \mathcal{A}_kA(\mathbf{a})\rangle\geq(1-\|W\|)\|\mathcal{P}_T^\bot\mathcal{A}_kA(\mathbf{a})\|_*.
\end{align*}
Since $\|W\|<1$, $\|\mathcal{A}_kA(\mathbf{y}+\mathbf{a})\|_*$ is greater than $\|\mathcal{A}_kA(\mathbf{y})\|_*$ unless $\mathcal{A}_kA(\mathbf{a})\in\mathcal{P}_T$. This can happen only when $\mathcal{A}_kA(\mathbf{a})=0$, because $\mathcal{A}_kA(\mathbf{a})\in\mathcal{P}_{\Theta}^\bot$ and $\mathcal{P}_{\Theta}^\bot\cap\mathcal{}P_T=\{0\}$. Since $\mathbf{a}=(1/k)A^T\mathcal{A}_k^*\mathcal{A}_kA(\mathbf{a})=0$, it follows that $\mathbf{y}$ is the unique minimizer.
\end{proof}
\subsubsection{Dual Certificate}
To construct the dual certificate, let's temporarily assume that
\begin{align*}
\|\mathcal{P}_T\mathcal{P}_{\Theta}^\bot\mathcal{P}_{T}\|<0.5.
\end{align*}
By Lemma~\ref{lem:basic:quivalence}, $\mathcal{P}_T\mathcal{P}_{\Theta}\mathcal{P}_{T}$ is invertible. Thus, we can construct $H$ as
\begin{align*}
H = \mathcal{P}_{\Theta}\mathcal{P}_T(\mathcal{P}_T\mathcal{P}_\Theta\mathcal{P}_T)^{-1}(UV^T).
\end{align*}
Then $H$ satisfies the dual conditions listed in Lemma~\ref{lem:dual}, because $\mathcal{P}_T\mathcal{P}_{\Theta}(H)=UV^T$ and it follows from Lemma~\ref{lem:optnorm:invpt} that
\begin{align*}
&\|\mathcal{}P_{T}^\bot\mathcal{P}_{\Theta}(H)\| \leq \|\mathcal{P}_{T}^\bot\mathcal{P}_{\Theta}\mathcal{P}_{T}(\mathcal{P}_{T}\mathcal{P}_{\Theta}\mathcal{P}_{T})^{-1}\|\|UV^T\|\\
&= \sqrt{\frac{1}{1-\|\mathcal{P}_{T}\mathcal{P}_{\Theta}^\bot\mathcal{P}_{T}\|}-1}<1.
\end{align*}

Hence, the proof is finished as long as $\|\mathcal{P}_T\mathcal{P}_{\Theta}^\bot\mathcal{P}_{T}\|<0.5$ is proven. Using arguments similar to the proof procedure of Lemma 5.11 in~\cite{liu:tpami:2019}, we have
\begin{align*}
\|\mathcal{P}_T\mathcal{P}_{\Theta}^\bot\mathcal{P}_{T}\|\leq{}\|\mathcal{P}_U\mathcal{P}_{\Theta}^\bot\mathcal{P}_U\|+\|\mathcal{P}_V\mathcal{P}_{\Theta}^\bot\mathcal{P}_V\|.
\end{align*}
In the rest of the proof, we will bound $\|\mathcal{P}_U\mathcal{P}_{\Theta}^\bot\mathcal{P}_U\|$ and $\|\mathcal{P}_V\mathcal{P}_{\Theta}^\bot\mathcal{P}_V\|$ individually.
\subsubsection{Bounding $\|\mathcal{P}_U\mathcal{P}_{\Theta}^\bot\mathcal{P}_U\|$}
Denote by $\mathrm{vec}(\cdot)$ the vector formed by stacking the columns of a matrix into a single column vector. For any $Z\in\mathbb{R}^{q\times{}k}$, we have
\begin{align*}
&\mathrm{vec}(\mathcal{P}_U\mathcal{P}_{\Theta}^\bot\mathcal{P}_U(Z)) = \\
&\sum_{i=1}^k(\mathbf{e}_i\mathbf{e}_i^T\otimes{}(UU^TT^{i-1}\tilde{A}(\Id_q-D)\tilde{A}^TT^{1-i}UU^T))\mathrm{vec}(Z)\\
&=(\Id_k\otimes{}UU^T)(\sum_{i=1}^k\mathbf{e}_i\mathbf{e}_i^T\otimes{}T^{i-1}\tilde{A}(\Id_q-D)\tilde{A}^TT^{1-i})\\
&(\Id_k\otimes{}UU^T)\mathrm{vec}(Z).
\end{align*}
Define $M$ as
\begin{align}\label{eq:matrix:m}
&M=\sum_{i=1}^k\mathbf{e}_i\mathbf{e}_i^T\otimes{}T^{i-1}\tilde{A}(\Id_q-D)\tilde{A}^TT^{1-i}=\\\nonumber
&\left[\begin{array}{ccc}
\hspace{-0.1in}\tilde{A}(\Id_q\hspace{-0.03in}-\hspace{-0.03in}D)\tilde{A}^T & &\\
&\hspace{-0.2in} T\tilde{A}(\Id_q\hspace{-0.03in}-\hspace{-0.03in}D)\tilde{A}^TT^{-1}\hspace{-0.2in}&\\
&\ddots&\\
&&\hspace{-0.3in}T^{k-1}\tilde{A}(\Id_q\hspace{-0.03in}-\hspace{-0.03in}D)\tilde{A}^TT^{1-k}\hspace{-0.1in}
\end{array}\right]\hspace{-0.08in}\in\hspace{-0.03in}\mathbb{R}^{kq\times{}kq}.
\end{align}
Since $(\Id_k\otimes{}UU^T)M(\Id_k\otimes{}UU^T)$ is block-diagonal, we have
\begin{align*}
&\|\mathcal{P}_U\mathcal{P}_{\Theta}^\bot\mathcal{P}_U\|=\max_{i\in[k]}\|UU^TT^{i-1}\tilde{A}(\Id_q-D)\tilde{A}^TT^{1-i}UU^T\|\\
&\leq(1-\rho_0)m\max_{i\in[k],j\in[m]}\|U^TT^{i-1}A\tilde{\mathbf{e}}_j\|_2^2,
\end{align*}
where $\tilde{\mathbf{e}}_j$ is the $j$th standard basis of $\mathbb{R}^m$. By Definition~\ref{defn:geconvcoherence},
\begin{align*}
&\|\mathcal{P}_U\mathcal{P}_{\Theta}^\bot\mathcal{P}_U\|\leq{}\frac{\mu_A(\mathbf{y})(1-\rho_0)mr}{q}<\frac{0.25k}{q}\leq0.25.
\end{align*}
\subsubsection{Bounding $\|\mathcal{P}_V\mathcal{P}_{\Theta}^\bot\mathcal{P}_V\|$}
Considering the vectorization of $\mathcal{P}_V\mathcal{P}_{\Theta}^\bot\mathcal{P}_V$, we have
\begin{align*}
&\mathrm{vec}(\mathcal{P}_V\mathcal{P}_{\Theta}^\bot\mathcal{P}_V(Z))=(VV^T\otimes{}\Id_q)\\
&(\sum_{i=1}^k\mathbf{e}_i\mathbf{e}_i^T\otimes{}T^{i-1}\tilde{A}(\Id_q-D)\tilde{A}^TT^{1-i})(VV^T\otimes{}\Id_q)\mathrm{vec}(Z)\\
&=(VV^T\otimes{}\Id_q)M(VV^T\otimes{}\Id_q)\mathrm{vec}(Z),
\end{align*}
where $M$ is the block-diagonal matrix defined in~\eqref{eq:matrix:m}. Denote by $M_{(i,i)}$ the $(i,i)$th block of $M$, and let $\bar{h}=(1-\rho_0)m$. Then we have
\begin{align*}
M_{(i,i)} = T^{i-1}A_{\bar{h}}A_{\bar{h}}^TT^{1-i}, \forall{}i\in[k],
\end{align*}
where $A_{\bar{h}}=[\mathbf{a}_1,\cdots,\mathbf{a}_{\bar{h}}]\in\mathbb{R}^{q\times{}\bar{h}}$ is formed by selecting $\bar{h}$ columns from $\tilde{A}$ such that $A_{\bar{h}}A_{\bar{h}}^T=\tilde{A}(\Id_q-D)\tilde{A}^T$. Let $Q=VV^T\otimes{}\Id_q$. It can be verified that $Q\in\mathbb{R}^{qk\times{}qk}$ is a block matrix with $k$ row partitions and $k$ column partitions, and its $(i,j)$th block, denoted as $Q_{(i,j)}$, is a $q\times{}q$ diagonal matrix given by $Q_{(i,j)} = \mathbf{v}_i^T\mathbf{v}_j\Id_q, \forall{}i,j\in[k]$, where $\mathbf{v}_i=V^T\mathbf{e}_i\in\mathbb{R}^r$ is the $i$th row vector of $V$. Denote
\begin{align*}
Q_{(:,j)} = \left[\begin{array}{c}
\mathbf{v}_1^T\mathbf{v}_j\Id_q\\
\mathbf{v}_2^T\mathbf{v}_j\Id_q\\
\vdots\\
\mathbf{v}_k^T\mathbf{v}_j\Id_q
\end{array}\right]\in\mathbb{R}^{kq\times{}q},\forall{j}\in{}[k].
\end{align*}
Then we have
\begin{align*}
&Q_{(:,i)}^TQ_{(:,j)} = \sum_{d=1}^k\mathbf{v}_d^T\mathbf{v}_i\mathbf{v}_d^T\mathbf{v}_j\Id_q\\
&=\sum_{d=1}^k\mathbf{v}_i^T(\mathbf{v}_d\mathbf{v}_d^T)\mathbf{v}_j\Id_q=\mathbf{v}_i^T\mathbf{v}_j\Id_q.
\end{align*}
For any $\mathbf{z}\in\mathbb{R}^{qk}$ with $\|\mathbf{z}\|_2^2=1$, we have
\begin{align*}
&\mathbf{z}^T(VV^T\otimes{}\Id_q)M(VV^T\otimes{}\Id_q)\mathbf{z} = \sum_{j=1}^k\|\mathbf{z}^TQ_{(:,j)}T^{j-1}A_{\bar{h}}\|_2^2\\
&=\sum_{j=1}^k\sum_{d=1}^{\bar{h}}(\mathbf{z}^TQ_{(:,j)}T^{j-1}\mathbf{a}_d)^2\\
&=\sum_{d=1}^{\bar{h}}\|\mathbf{z}^T[Q_{(:,1)}\mathbf{a}_d,Q_{(:,2)}T\mathbf{a}_d,\cdots,Q_{(:,k)}T^{k-1}\mathbf{a}_d]\|_2^2\\
&\leq\sum_{d=1}^{\bar{h}}\|[Q_{(:,1)}\mathbf{a}_d,Q_{(:,2)}T\mathbf{a}_d,\cdots,Q_{(:,k)}T^{k-1}\mathbf{a}_d]\|^2.
\end{align*}
So, it remains to bound the operator norm of $\tilde{P}_d=[Q_{(:,1)}\mathbf{a}_d,Q_{(:,2)}T\mathbf{a}_d,\cdots,Q_{(:,k)}T^{k-1}\mathbf{a}_d], \forall{}d\in[\bar{h}]$. Let $P_d=\tilde{P}_d^T\tilde{P}_d\in\mathbb{R}^{k\times{}k}$. It can be seen that the $(i,j)$th entry of $P_d$, denoted as $d_{ij}$, is given by
\begin{align*}
d_{ij} = \mathbf{a}_d^TT^{1-i}Q_{(:,i)}^TQ_{(:,j)}T^{j-1}\mathbf{a}_d=(\mathbf{v}_i^T\mathbf{v}_j)(\mathbf{a}_d^TT^{j-i}\mathbf{a}_d).
\end{align*}
Notice that $T^{-i}= T^{q-i}$, $\forall{}i$, due to the property of circular convolution. By~\eqref{eq:convcoherence} and Definition~\ref{defn:basiscoherence},
\begin{align*}
&\sum_{j=1}^k|d_{ij}|\leq\frac{\hat{\mu}_2(A\mathbf{y})r}{k}\sum_{j=1}^k|\mathbf{a}_d^TT^{j-i}\mathbf{a}_d|\\
&\leq\frac{\bar{\mu}(\mathbf{a}_d)\hat{\mu}_2(A\mathbf{y})r}{k},\forall{}i\in[k],
\end{align*}
from which it follows that $\|\tilde{P}_d\|^2=\|P_d\|\leq\bar{\mu}(\mathbf{a}_d)\hat{\mu}_2(A\mathbf{y})$ $r/k$, and which simply leads to
\begin{align*}
\|\mathcal{P}_V\mathcal{P}_{\Theta}^\bot\mathcal{P}_V\|\leq\frac{\bar{\mu}(A)\hat{\mu}_2(A\mathbf{y})(1-\rho_0)mr}{k}<0.25.
\end{align*}
\subsection{Proof of Theorem~\ref{main:thm:noisy}}
Let $\mathbf{c} = \mathbf{x} - \mathbf{y}$ and $C =\mathcal{A}_kA(\mathbf{c})$. By triangle inequality, $\|\mathcal{P}_{\Omega}(\mathbf{c})\|_2\leq2\epsilon$. Thus,
\begin{align*}
&\|\mathcal{P}_{\Theta}(C)\|_F^2 = \|\mathcal{P}_{\Theta}\mathcal{A}_kA(\mathbf{c})\|_F^2\\
&=\|\mathcal{A}_kA\mathcal{P}_{\Omega}(\mathbf{c})\|_F^2= k\|A\mathcal{P}_{\Omega}(\mathbf{c})\|_F^2\leq4k\epsilon^2.
\end{align*}
In the rest of the proof, we will bound $\|\mathcal{P}_{\Theta}^\bot(C)\|_F$. To this end, define $H$ and $W$ in the same way as in the proof to Theorem~\ref{main:thm:noiseless}. Provided that $\rho_0>1-0.22k/(\mu{}mr)$, it can be proven that $\|W\|=\|\mathcal{P}_{T}^\bot\mathcal{P}_{\Theta}(H)\|<0.9$, thereby $\|H\|\leq1.9$. Since $\mathbf{x}=\mathbf{y}+\mathbf{c}$ is an optimal solution to~\eqref{eq:lbcnnm:noisy}, it follows that
\begin{align*}
&0\geq\|\mathcal{A}_kA(\mathbf{y}+\mathbf{c})\|_* - \|\mathcal{A}_kA(\mathbf{y})\|_*\geq{}\langle{}\mathcal{P}_{\Theta}(H)-W+F,C\rangle\\
&\geq(1-\|W\|)\|\mathcal{P}_T^\bot(C)\|_* + \langle{}\mathcal{P}_{\Theta}(H),C\rangle\\
&\geq0.1\|\mathcal{P}_T^\bot(C)\|_* + \langle{}\mathcal{P}_{\Theta}(H),C\rangle.
\end{align*}
Hence,
\begin{align*}
&\|\mathcal{P}_T^\bot(C)\|_*\leq-10\langle{}\mathcal{P}_{\Theta}(H),\mathcal{P}_{\Theta}(C)\rangle\\
&\leq10\|\mathcal{P}_{\Theta}(H)\|\|\mathcal{P}_{\Theta}(C)\|_*\leq38k\epsilon,
\end{align*}
which leads to $\|\mathcal{P}_T^\bot(C)\|_F\leq38k\epsilon$. Thus we have
\begin{align*}
&\|\mathcal{P}_{T}^\bot\mathcal{P}_{\Theta}^\bot(C)\|_F\leq\|\mathcal{P}_{T}^\bot(C)\|_F+\|\mathcal{P}_{T}^\bot\mathcal{P}_{\Theta}(C)\|_F\\
&\leq(38k+2\sqrt{k})\epsilon.
\end{align*}
In addition, we have
\begin{align*}
&\|\mathcal{P}_{\Theta}\mathcal{P}_{T}\mathcal{P}_{\Theta}^\bot(C)\|_F^2=\langle{}\mathcal{P}_{\Theta}\mathcal{P}_{T}\mathcal{P}_{\Theta}^\bot(C),\mathcal{P}_{\Theta}\mathcal{P}_{T}\mathcal{P}_{\Theta}^\bot(C)\rangle\\
&=\langle\mathcal{P}_{T}\mathcal{P}_{\Theta}\mathcal{P}_{T}\mathcal{P}_{\Theta}^\bot(C),\mathcal{P}_{T}\mathcal{P}_{\Theta}^\bot(C)\rangle\\
&\geq(1-\|\mathcal{P}_{T}\mathcal{P}_{\Theta}^\bot\mathcal{P}_{T}\|)\|\mathcal{P}_{T}\mathcal{P}_{\Theta}^\bot(C)\|_F^2\geq\frac{1}{2}\|\mathcal{P}_{T}\mathcal{P}_{\Theta}^\bot(C)\|_F^2,
\end{align*}
from which it follows that
\begin{align*}
&\|\mathcal{P}_{T}\mathcal{P}_{\Theta}^\bot(C)\|_F^2\leq2\|\mathcal{P}_{\Theta}\mathcal{P}_{T}\mathcal{P}_{\Theta}^\bot(C)\|_F^2\\
&=2\|\mathcal{P}_{\Theta}\mathcal{P}_{T}^\bot\mathcal{P}_{\Theta}^\bot(C)\|_F^2\leq2(38k+2\sqrt{k})^2\epsilon^2.
\end{align*}
Via combining the above justifications, we have
\begin{align*}
&\|C\|_F\leq\|\mathcal{P}_{T}\mathcal{P}_{\Theta}^\bot(C)\|_F + \|\mathcal{P}_{T}\mathcal{P}_{\Theta}(C)\|_F\\
&+\|\mathcal{P}_{T}^\bot(C)\|_F\leq(\sqrt{2}+1)(38k + 2\sqrt{k})\epsilon.
\end{align*}
Finally, the proof is finished by $\|\mathbf{c}\|_2 = \|A\mathbf{c}\|_2=\|C\|_F/\sqrt{k}$.
\begin{figure}[h!]
\begin{center}
\includegraphics[width=0.49\textwidth]{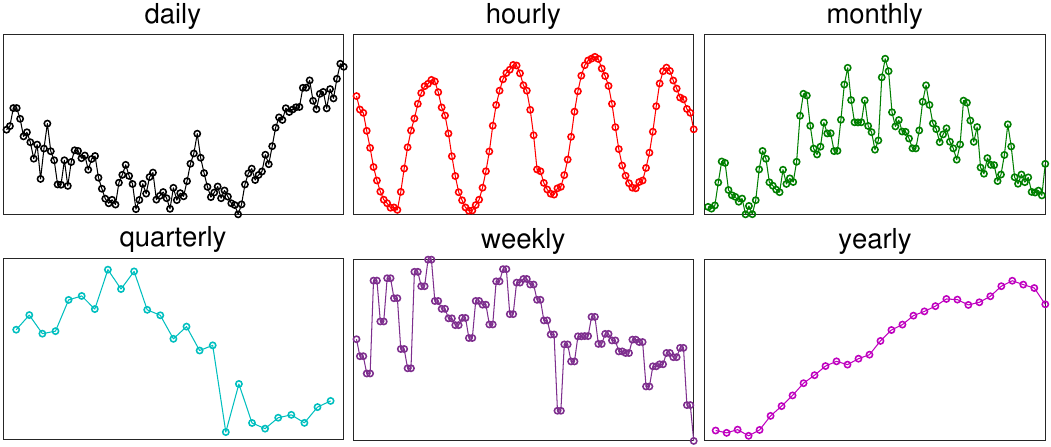}\vspace{-0.15in}
\caption{Examples of the time series in M4.}\label{fig:m4}
\end{center}
\end{figure}
\section{Experiments}\label{sec:exp}
All experiments are conducted on the Matlab 2019a platform. The source codes are available at https://github. com/gcliu1982/LbCNNM.
\subsection{Data and Setup}
\subsubsection{Experimental Data}
We consider for experiments two datasets: TSDL~\cite{TSDL:dataset} and M4~\cite{M4:ijf:2020}. TSDL is consist of 618 univariate and 38 multivariate series of sampling-frequencies ranging from 0.1 to 365. For the ease of conducting experiments, we separate each multivariate series into multiple univariate ones, obtaining 761 univariate series in total. We set the forecast horizon $h$ as follows: $h=6$ for the series whose sampling-frequency is lower than 2, $h=24$ for the series with sampling-frequency higher than 100, and $h=12$ for the rest. We consider only the series longer than $10h$---the training part is longer than $9h$, thereby 542 series is selected for experiments. To avoid possible problems in calculating various evaluation metrics, the values in each series are scaled to be greater than or equal to 10 by simply adding a constant.

M4 is a large-scale dataset containing 100,000 series from six sampling-frequency categories, including 4227 Daily series, 414 Hourly, 48,000 Monthly, 24,000 Quarterly, 359 Weekly, and 23,000 Yearly. Figure~\ref{fig:m4} shows some examples selected freely from M4. In the original M4 dataset, the values have been scaled to prevent numbers smaller than 10, and the forecast horizon for each series has been configured in the testing set: $h=48$ for the Hourly series, $h=14$ for Daily, $h=18$ for Monthly, $h=8$ for Quarterly, $h=13$ for Weekly, and $h=6$ for Yearly.

Table~\ref{tb:info} summarizes some information about the series in TSDL and M4. The spectral-frequency, denoted as $\hat{f}$, is found by maximizing the Power Spectral Density (PSD) of the training sequence. For a sequence of length $l$, its spectral-frequency $\hat{f}$, which ranges from 0 to $l/2$, describes how rapidly its pattern oscillates. The stationarity level (ranging from 0.1\% to 100\%) is calculated by the Augmented Dickey-Fuller (ADF) test. As shown in Table~\ref{tb:info}, most of the series in M4 is short ($l<10h$), low-frequency ($\hat{f}<5$) and non-stationary. Here, we calculate the stationarity level as $(1-p_\mathrm{value})\times100\%$ such that higher stationarity level means that the series is more likely to be stationary. Nevertheless, to reject the null hypothesis that the series is non-stationary at significance level 5\%, the stationarity level has to be higher than 95\%. So, as can be seen from Table~\ref{tb:info}, the overwhelming majority of M4, as well as TSDL, is the non-stationary series.
\begin{table}[h!]
\caption{Information about the time series used in our experiments.}\label{tb:info}\vspace{-0.2in}
\begin{center}
\renewcommand\tabcolsep{3pt}
\begin{tabular}{c|c|cc|cc|cc}
\toprule[2pt]
           &no.  &\multicolumn{2}{c}{length}  & \multicolumn{2}{|c}{spectral-} & \multicolumn{2}{|c}{stationarity} \\
           &of &\multicolumn{2}{c}{($l/h$)}  & \multicolumn{2}{|c}{frequency} & \multicolumn{2}{|c}{level(\%)} \\\cline{3-8}
Testbeds   &series  &mean   & std       &mean& std      & mean& std\\\hline
TSDL     &452       &85.01  & 233.45    &23.71& 56.74   & 67.58& 33.41\\
Hourly (M4)   &414       &17.78  & 2.66      &36.44& 8.98    & 74.00& 19.72\\
Daily (M4)    &4227       &168.38 & 125.46    &1.75& 0.93     & 34.62& 26.54\\
Weekly (M4)   &359        &78.61  &54.39      &4.93& 17.92    & 49.33& 36.64\\
Monthly (M4)   &48,000    &12.01  & 7.63      &4.83&10.21     & 46.40& 34.79\\
Quarterly (M4)  &24,000   &11.53  &6.39       &4.12& 8.63     & 28.24& 30.69\\
Yearly (M4)     &23,000   &5.22   &4.08       &1.60& 1.00     &16.66  & 27.61\\\toprule[2pt]
\end{tabular}\vspace{-0.15in}
\end{center}
\end{table}
\subsubsection{Evaluation Metrics}
To measure the forecasting accuracy of various methods, we first consider the Symmetric Mean Absolute Percentage Error (sMAPE)~\cite{Makri:1993:ijf} widely used in the community of time series forecasting. The formulate for calculating sMAPE is the following:
\begin{align*}
\mathrm{sMAPE} = \frac{2}{h}\sum_{i=1}^h\frac{|y_i-\hat{y}_i|}{|y_i|+|\hat{y}_i|}\times100\%,
\end{align*}
where $\hat{y}_i$ denotes an estimate to $y_i$. Since the estimate $\hat{y}_i$ itself appears in the denominator, sMAPE tends to place more penalty on underestimates (i.e., $|\hat{y}_i|<|y_i|$) and suppress the impact of overestimates (i.e., $|\hat{y}_i|>|y_i|$). For the sake of comprehensive evaluation, we also consider the Normalized Root Mean Square Error (NRMSE) commonly used in machine learning and signal processing, which is computed as follows:
\begin{align*}
\mathrm{NRMSE} = \frac{\sqrt{h\sum_{i=1}^h(y_i-\hat{y}_i)^2}}{\sum_{i=1}^h|y_i|}\times100\%.
\end{align*}
Unlike sMAPE, which expresses sympathy for the methods that produce large overestimation errors occasionally, NRMSE is cruel in treating large overestimates. For example, let $[10; 1]$ be an estimate to $[1; 1]$. Then the error rate is 636.39\% in the light of NRMSE, but only 81.81\% in terms of sMAPE. When the error rate is very small, sMAPE can coincide with NRMSE , e.g., $\mathrm{sMAPE}=0$ implies $\mathrm{NRMSE}=0$ and vice versa. But, in general cases, these two metrics may give different assessments.
\subsubsection{Parametric Settings}\label{sec:modelselection} LbCNNM has two models,~\eqref{eq:lbcnnm:exact} and~\eqref{eq:lbcnnm:noisy}, but the exact model in~\eqref{eq:lbcnnm:exact} is never used in practice. For the ease of implementation, we shall not try to solve the problem in~\eqref{eq:lbcnnm:noisy} directly, but instead consider an equivalent problem as follows:
\begin{align}\label{eq:lbcnnm:noisy2}
\min_{\mathbf{x}}\|\mathcal{A}_k(A\mathbf{x})\|_*+\frac{\lambda{}k}{2}\|\mathcal{P}_{\Omega}(\mathbf{x}-\hat{\mathbf{y}})\|_2^2,
\end{align}
where $\lambda>0$ is a parameter. The above problem can be solved by the standard ADMM algorithm. As the nuclear norm is already good at handling heavy-tailed data, there is no need to tune $\lambda$ and thus we consistently use $\lambda=1000$ in all experiments. The kernel size $k=\beta{}q$ is not involved in the training procedure and we consistently use $k=0.5q$ during testing. In fact, under the background of point forecasting, it seems unavailing to try finding another configuration better than $k=0.5q$.

Under the standard setup that one is asked to predict the future part of a given sequence, as we have pointed out in Section~\ref{sec:gematrix:pcp}, the model size $m$ is an unknown parameter waiting to be determined. In the ideal case where the training sequence, $\tilde{\mathbf{y}}\in\mathbb{R}^l$, is very very long (e.g., $l\rightarrow{}+\infty$), one just needs to set $m$ to be large enough, e.g., $m=10h$. Unfortunately, the series in reality always has a finite length and is often very short, e.g., more than half of the series in M4 is shorter than $10h$. In this case, $m$ becomes a very influential hyper-parameter and needs be estimated carefully. First of all, for the generation matrix $G_0\in\mathbb{R}^{m\times{}n_0}$ ($n_0=l-m+1$) to be low-rank, $m$ cannot be too small. On the other hand, for the learnt transform matrix $A$ to generalize well (see the discussions around~\eqref{eq:generalization}), the following quantity, called \emph{sample-to-dimension ratio} (SDR) and defined as
\begin{align*}
\mathrm{SDR} = \frac{n_0}{m}=\frac{l-m+1}{m},
\end{align*}
should be large enough, which requires $m$ to be sufficiently small. Such a contradiction reveals that estimating $m$ is in fact a challenging \emph{model selection} problem, and there is no magic to ascertain the best choice in general cases. In this work, we shall adopt some heuristic approaches to estimate $m$. Concretely, for each series, we select $m$ from the range between $2h$ to $10h$ via \emph{regularized empirical risk minimization}:
\begin{align*}
\min_{2h\leq{}m\leq{}10h} \mathrm{EGE} + \gamma\times\mathrm{SpEnt}(G_0), \textrm{ s.t. } \mathrm{SDR}\geq{}\tau,
\end{align*}
where $\mathrm{EGE}$ is the empirical generalization error of LbCNNM obtained from $b$-fold cross validations ($b=7.5 + 2.5\tanh(l/h - 10)$), $\mathrm{SpEnt}$ is the spectral entropy defined in~\eqref{eq:ent}, and $\gamma$ and $\tau$ are two parameters. The regularization parameter $\gamma{}$ is set as
\begin{align*}
&\gamma = \left\{\begin{array}{cc}
0.4, & \textrm{if } \hat{f}> 5,\\
0, &\textrm{otherwise,}
\end{array}\right.
\end{align*}
where $\hat{f}$ is the spectral-frequency of the training sequence $\tilde{\mathbf{y}}$. The parameter $\tau$, which gives the upper bound of the model size $m$, is related to the sequence length $l/h$, as well as the spectral-frequency $\hat{f}$. According to some empirical analysis, we set it as follows:
\begin{align*}
&\hspace{-0.03in}\tau\hspace{-0.03in}=\hspace{-0.03in} \left\{\begin{array}{ll}
\hspace{-0.1in}4+\tanh(l/h-25) + \tanh(5-\hat{f}), & \hspace{-0.05in}\textrm{if } l> 13h,\\
\hspace{-0.1in}l/(h(5.5+\tanh(l/h-8.5)\\
\hspace{-0.1in}+0.5\tanh(\hat{f} - 4 - \tanh(l/h - 8.5)) &\\
\hspace{-0.1in}+ 0.3\tanh(4 + \tanh(l/h - 8.5) - \hat{f}))), & \hspace{-0.05in}\textrm{if } 5h<l\leq13h,\\
\hspace{-0.1in}0.7 + 0.05\tanh(l/h - 3.8) - 0.15(1 +\\
\hspace{-0.1in} \tanh(3.3 - l/h))+(0.2 + 0.05\tanh & \\
\hspace{-0.1in}(l/h - 3.8))\tanh(2.5 - \hat{f}), & \hspace{-0.05in}\textrm{if } l\leq5h.
\end{array}\right.
\end{align*}
Since $\tau$ is at most 6, the SDR constraint has no function whenever the sequence is long enough, namely $l>70h$. The above choices are found by first partitioning 10,000 series---selected randomly from M4---into different groups according to the length $l/h$ and the spectral-frequency $\hat{f}$, then finding manually the optimal parametric value for each group, and finally fitting those manually determined values into few smooth functions to avoid over-fitting.

During the process of estimating the model size $m$, we wound suggest not to test the generalization performance of Algorithm~\ref{alg:learnA}, but instead to consider only Algorithm~\ref{alg:learnA:pca} (with $Y=G_0$), which is much more computationally efficient. In fact, this is not just for the sake of efficiency, but also for relieving the issue of over-fitting. To further reduce the computation budgets, the kernel size is set as $k=q$, because in this case the optimization procedure of program~\eqref{eq:lbcnnm:noisy2} can be significantly sped up by Fast Fourier Transform (FFT)~\cite{fft:1998}. There is no loss either to use $k=q$ during the stage of estimating $m$, as it seems superfluous to include the kernel size $k$ into the training procedure.
\subsubsection{Data Argumentation and Model Combination}\label{sec:argumentation}
Data argumentation is a widely used strategy for improving the performance of learning-based forecasting methods~\cite{Wen:2020:arxiv}. For the proposed LbCNNM, there is a very intuitive reason for why data argumentation is helpful. As we have pointed out in Section~\ref{sec:learnA:pipe}, learning successfully the transformation $A$---such that the training samples are transformed to be convolutionally low-rank or approximately so---does not necessarily lead to accurate forecasts. This is because, especially when the training sequence is short, there is no guarantee that the target $\mathbf{y}$, which is unknown, is represented in a convolutionally low-rank manner.

To improve the generalization performance of LbCNNM, we would like to suggest a heuristic approach, the motivation behind which is rather straightforward: Just imagine what if the ground truth of the target $\mathbf{y}$ has been included as a training sample for learning $A$. If so, $\mathbf{y}$ will be contained by the generalizable space $\mathcal{S}_t(Y)$ with $t=1$, and thus LbCNNM is very likely to provide accurate forecasts. Although $\mathbf{y}$ is unknown, we can use the existing forecasting methods to estimate it. Hereafter, an estimate to the target $\mathbf{y}$ is also referred to as a \emph{pseudo-sample}. There are a great many avenues for generating pseudo-samples, as at least tens of thousands of forecasting methods have been established in the literature. As an example, in this work we shall consider only some simple methods, including Average, Drift, LSR, CNNM and ExpS (single).

Regarding Average, we take only one pseudo-sample, denoted as $\hat{\mathbf{y}}^{\mathrm{avg}}=[\hat{y}_1^{\mathrm{avg}},\cdots,\hat{y}_m^{\mathrm{avg}}]\in\mathbb{R}^m$, that estimates the target $\mathbf{y}$ by fitting a horizontal line into the last $h$ observations:
\begin{align*}
\hat{y}_i^{\mathrm{avg}} = \frac{\sum_{j=l-h+1}^l\tilde{y}_j}{h}, i=1,\cdots,m,
\end{align*}
where $\tilde{y}_j$ is the $j$th observation of the training sequence $\tilde{\mathbf{y}}$. As for Drift and LSR, we set a window of size $w_s=(2.75+0.25\tanh(10(l/h - 5.5)))h$ for them to generate multiple pseudo-samples. The pseudo-samples from Drift and LSR are all straight lines and can be described by a unified formula as follows:
\begin{align*}
\hat{y}_i^{\mathrm{drift}} \textrm{ or } \hat{y}_i^{\mathrm{LSR}}= \alpha_1(l + h - m + i) + \alpha_2, i=1,\cdots,m,
\end{align*}
where $\alpha_1$ and $\alpha_2$ are the slope and intercept of a line, respectively. As regards Drift, which draws a line between the last and $j$th observations ($l-w_s+1\leq{}j\leq{}l-1$), there are $w_s-1$ estimates to $(\alpha_1, \alpha_2)$, resulting in $w_s-1$ pseudo-samples. The LSR method, which fits a line into the last $j$ observations ($3\leq{}j\leq{}w_s$), generates $w_s-2$ pseudo-samples. So, Drift and LSR provide in total $2w_s-1$ pseudo-samples. But we use only those potentially helpful, namely we discard the pseudo-samples whose fitting errors with respect to the last $h$ observations are greater than a threshold defined as follows:
\begin{align*}
&e_{\mathrm{th}}=\min(2e_{\mathrm{max}}-e_0, e_0) \textrm{ with } e_{\mathrm{max}} = 0.325+\\
&0.025\tanh(10(3-l/h))+ 0.05\tanh(f_{\mathrm{th}}-\hat{f})\textrm{ and }\\
&f_{\mathrm{th}} =3.75+1.25\tanh(l/h-5)-2.5 + 2.5\tanh(16-l/h),
\end{align*}
where $e_0=\max(\mathrm{Err}(\hat{\mathbf{y}}^{\mathrm{avg}}), \mathrm{Err}(\hat{\mathbf{y}}_h^{\mathrm{drift}}))$, $\mathrm{Err}(\hat{\mathbf{y}}^{\mathrm{avg}})$ is the fitting error of $\hat{\mathbf{y}}^{\mathrm{avg}}$ with respect to the last $h$ observations---$\mathrm{Err}(\hat{\mathbf{y}}^{\mathrm{avg}})$ is indeed a good estimate to the generalization error of $\hat{\mathbf{y}}^{\mathrm{avg}}$, and $\mathrm{Err}(\hat{\mathbf{y}}_h^{\mathrm{drift}})$ is the fitting error (with respect to the last $h$ observations) of the pseudo-sample generated by drawing a line between the last and $(l-h+1)$th observations. Putting all the remaining pseudo-samples together, we obtain for training a new matrix, denoted as $G_s$, of dimension $m\times{}n_s$ ($1\leq{}n_s\leq{}2w_s$). It is easy to see that the rank of this matrix is bounded from above by 2; that is, $\rank{G_s}\leq2$.

Although arguably fails in many cases, CNNM is still useful and, in fact, might have a unique value while facing short, high-frequency series---such series is rare in M4 though. Let $\hat{\mathbf{y}}^{\mathrm{cnnm}}\in\mathbb{R}^{m}$ be CNNM's estimate to $\mathbf{y}$, then we generate $m$ pseudo-samples via
\begin{align*}
T^{i-1}\hat{\mathbf{y}}^{\mathrm{cnnm}}, i = 1,\cdots,m,
\end{align*}
where $T$ is a permutation matrix defined as in~\eqref{eq:convmtx}. Again, we use only the pseudo-samples whose fitting errors with respect to the last $h$ observations are smaller than a certain threshold. Empirically, we find that a good error threshold for choosing the pseudo-samples from CNNM is simply $e_0=\max(\mathrm{Err}(\hat{\mathbf{y}}^{\mathrm{avg}}), \mathrm{Err}(\hat{\mathbf{y}}_h^{\mathrm{drift}}))$. In this way, we obtain another matrix, denoted as $G_c\in\mathbb{R}^{m\times{}n_c}$ ($1\leq{}n_c\leq{}m$). Unlike the pseudo-samples in $G_s$, which are all straight lines, $\hat{\mathbf{y}}^{\mathrm{cnnm}}$ is often a nonlinear curve, and thus $G_c$ may not be strictly low-rank. Fortunately, due to the nature of CNNM, $G_c$ often meets the property of approximate low-rankness. While using $Y=G_c$ to learn the transformation $A$, LbCNNM mostly falls back to CNNM.

We also consider ExpS, which is one of the most widely used methods in time series forecasting. This method depends on a key parameter, denoted as $\alpha_{\mathrm{exps}}$, that balances the weights between the smoothed statistics and the original observations. Once $\alpha_{\mathrm{exps}}$ is determined, ExpS gives a unique estimate to the target $\mathbf{y}$. Thus, we can generate multiple pseudo-samples by a collection of parameters configured as follows:
\begin{align*}
&\alpha_{\mathrm{exps}}=\left\{\begin{array}{ll}
0.05, & \textrm{if  } \hat{f} > 10,\\
\{0.05, 0.1\}, & \textrm{if  } 5<\hat{f} \leq 10,\\
\{0.5, 0.55, \cdots, 1\}, & \textrm{if  } 2.5 < \hat{f} \leq5,\\
\{0.7, 0.75, \cdots, 1\}, & \textrm{if  } 1.25 < \hat{f} \leq2.5,\\
\{0.9, 0.95, 1\}, &\textrm{if  } \hat{f} \leq1.25.\\
\end{array}\right.
\end{align*}
The data matrix formed by the pseudo-samples from ExpS is denoted as $G_e\in\mathbb{R}^{m\times{}n_e}$, where $1\leq{}n_e\leq11$. Since the estimates provided by ExpS are nonlinear curves, $G_e$ is often not strictly low-rank. But it is unnecessary to be too worried about this, as the pseudo-samples in $G_e$ are very few (at most 11) and, even more, the rich redundancy among the pseudo-samples generated by a collection of adjacent parameters tends to produce low-rankness. While using $Y=G_e$ to learn $A$, LbCNNM is almost identical to ExpS.

Now, we have in hand four matrices, $G_0$, $G_c$, $G_s$ and $G_e$, each of which has certain strengths and weaknesses. It remains to combine them together. A simple approach for combining is the concatenation, namely $Y=[G_0,G_c,G_s,G_e]\in\mathbb{R}^{m\times{}n}$ ($n=n_0+n_c+n_s+n_e$). In fact, due to the power of LbCNNM, concatenation performs pretty well when combining $G_0$ with $G_s$; this is probably because $\rank{G_s}\leq2$. However, the process is not so smooth when combining with $G_c$ and $G_e$, which are not strictly low-rank. In particular, since CNNM's overall performance is very poor, it is possibly harmful to use $G_c$. Based on some priors, we design a combination, denoted as $G_0\oplus{}G_c\oplus{}G_s\oplus{}G_e$, as follows:
\begin{align*}
&G_0\oplus{}G_c\oplus{}G_s\oplus{}G_e=\\
&\left\{\begin{array}{ll}
[G_0,G_s, G_e], & \textrm{if  } \mathrm{SpGini}(G_0) > g_{\mathrm{th}},\\\relax
[G_0, G_c, G_s, G_e], & \textrm{otherwise, }
\end{array}\right.
\end{align*}
where $g_{\mathrm{th}}$ is a threshold parameter set as
\begin{align*}
&g_{\mathrm{th}} = 0.8 + 0.05\tanh(12 - l/h) + 0.1\tanh(5(4 - l/h)) \\
&+ (0.1 + 0.1\tanh(l/h - 12))\tanh(f_{\mathrm{th}}  - \hat{f}) \textrm{ with }\\
&f_{\mathrm{th}} = 6.25 + 1.25\tanh(l/h - 4) + 2.5\tanh(l/h - 12),
\end{align*}
Intuitively, the above model selection criterion is simply to discard $G_c$ whenever $\mathrm{SpGini}(G_0)$ is larger than certain threshold, i.e., $G_0$ is not far from being low-rank.

Besides the simple methods used in this work, some more advanced forecasting methods developed in the literature should be also applicable for data argumentation. In that sense, the proposed LbCNNM can be used as a general tool for model combination, leaving much space for further improvement.
\begin{table}[h!]
\caption{Ablation studies on M4, using (averaged) sMAPE as the evaluation metric.}\label{tb:as:datamodel}\vspace{-0.15in}
\begin{center}
\renewcommand\tabcolsep{1pt}
\begin{tabular}{l|c|c|c}\toprule[2pt]
Data                   &Learning                 &45,352 series  &all 100,000\\
Model                  &Method                  &($l>10h$) & series \\ \hline
$Y=G_0$                     & Algorithm~\ref{alg:learnA:pca} &10.52             &18.39       \\\hline
$Y=G_0$                     & Algorithm~\ref{alg:learnA}     &9.63              &16.76       \\
$Y=G_s$                     & Algorithm~\ref{alg:learnA}     &10.58             &15.30       \\
$Y=G_c$                     & Algorithm~\ref{alg:learnA}     &12.57             &20.08       \\
$Y=G_e$                     & Algorithm~\ref{alg:learnA}     &9.02              &14.08       \\\hline
$Y=[G_0,G_s]$               & Algorithm~\ref{alg:learnA}     &9.17              &14.17       \\
$Y=[G_0,G_c]$               & Algorithm~\ref{alg:learnA}     &9.61              &16.03       \\
$Y =[G_0,G_e]$              & Algorithm~\ref{alg:learnA}     &9.04              &14.17      \\\hline
$Y=[G_0,G_c,G_s]$           & Algorithm~\ref{alg:learnA}     &9.11              &13.84       \\
$Y=[G_0,G_c,G_e]$           & Algorithm~\ref{alg:learnA}     &9.09              &14.21       \\
$Y=[G_0,G_s,G_e]$           & Algorithm~\ref{alg:learnA}     &8.73              &13.36      \\\hline
$Y=[G_0,G_c, G_s,G_e]$      & Algorithm~\ref{alg:learnA}     &8.72              &13.29       \\
$Y=G_0\oplus{}G_c\oplus{}G_s\oplus{}G_e$ & Algorithm~\ref{alg:learnA}    &\textbf{8.69}         &\textbf{13.24} \\\bottomrule[2pt]
\end{tabular}\vspace{-0.1in}
\end{center}
\end{table}
\subsection{Experimental Results}
\subsubsection{Ablation Studies}
Table~\ref{tb:as:datamodel} shows the performance of LbCNNM under various data models and different learning methods. The first find from M4 is that PCP (Algorithm~\ref{alg:learnA}) is distinctly better than PCA (Algorithm~\ref{alg:learnA:pca}), revealing the importance of modeling the dynamics possibly existing in time series. A somewhat ``weird'' phenomenon is that, over all 100,000 series, using alone the linear pseudo-samples in $G_s$ is even better than using the generation matrix $G_0$ extracted from the training sequence. This is partially due to the fact that the majority of M4 is the short series, as confirmed by that $Y=G_0$ performs better than $Y=G_s$ on 45,352 relatively long series. Notably, as can be seen from Table~\ref{tb:as:datamodel}, the concatenation of multiple data models is often better than the best of them, demonstrating the power of LbCNNM in model combination.

The specifically designed combination, $G_0\oplus{}G_c\oplus{}G_s\oplus{}G_e$, is slightly better than the simple concatenation, $[G_0,G_c, G_s,G_e]$. The messages from this phenomenon are two-fold. First, the better performance of $Y=G_0\oplus{}G_c\oplus{}G_s\oplus{}G_e$ over $Y=[G_0,G_c, G_s,G_e]$ suggests that it might be inaccurate to choose randomly some candidates from the existing methods and combine them together by LbCNNM. For model combination to make improvement, it is indeed necessary that the individual models do have complementary advantages. On the other hand, the difference is very mild, actually, illustrates that LbCNNM is quite intelligent in distinguishing the useful information from the harmful one---note that the overall performance of using $Y=G_c$ alone is very poor.

While replacing the $\ell_1$ loss in~\eqref{eq:learnB} by $\ell_2$---in this case the problem has a closed-form solution, LbCNNM with $Y=G_0\oplus{}G_c\oplus{}G_s\oplus{}G_e$ achieves an average sMAPE of 13.33\%, which is close to the 13.24\% produced by $\ell_1$ loss. This is due to the fact that the sparsity of $BY$ is mainly derived from the sparse matrix $E$ rather than the $\ell_1$ loss function.
\subsubsection{Comparison Results}
Besides the methods that have been used for data argumentation in Section~\ref{sec:argumentation}, we also consider for comparison the classic Auto-Regressive Moving Average (ARMA), as well as the popular Recurrent Neural Network (RNN). The results are shown in Table~\ref{tb:tsdl}, Table~\ref{tb:m4:category} and Table~\ref{tb:m4:cn}. According to NRMSE, LbCNNM surpasses the competing methods on all testbeds except the Daily category, on which all methods seem able to provide very accurate forecasts. In the light of sMAPE, LbCNNM's overall performance on M4 is 3.9\% better than the Holt Benchmark that is the 42th best among thousands of results submitted to M4 Competition~\cite{M4:ijf:2020}. In terms of NRMSE, LbCNNM outperforms by 11.6\% the Naive Benchmark, which overwhelmingly beats RNN by 32.8\% (according to sMAPE) and achieves 45th in M4 Competition. These results confirm the effectiveness of LbCNNM. What is more, it is entirely possible to further reduce the error rates by following the ideas introduced in Section~\ref{sec:argumentation}.

Remarkably, while comparing to the competing methods---except CNNM---in terms of NRMSE, LbCNNM outperforms largely the most close baseline by 57.9\% on the Hourly category. This, indeed, is not accidental, as similar results also appear on TSDL, which shares similar characteristics with Hourly (see Table~\ref{tb:info}). So, we would recommend LbCNNM for high-frequency series. Moreover, as we can see from Table~\ref{tb:m4:cn}, the performance of all methods declines as the stationarity level goes up. This is not strange, as the most stationary series is probably the white noise series, which is apparently unpredictable.
\begin{table}[h!]
\caption{Comparison results on TSDL (542 series).}\label{tb:tsdl}\vspace{-0.2in}
\begin{center}
\begin{tabular}{c|cc|cc}
\toprule[2pt]
            & \multicolumn{2}{c|}{Evaluation Metrics} &\multicolumn{2}{c}{Computational Time (secs)} \\\cline{2-5}
Methods     &sMAPE & NRMSE &  Training & Testing\\\hline
Naive       &29.46 & 39.42&0 & $1.19\times{}10^{-6}$\\
Average     &27.50 & 39.08&0 & $1.14\times{}10^{-6}$\\
Drift       &37.77 & 56.72& $1.27\times10^{-5}$ & $3.79\times10^{-6}$\\
LSR         &38.22 & 59.45& $2.63\times10^{-5}$          &  $4.04\times10^{-6} $          \\
ExpS        &26.79 & 35.61 & $0.0051$ & $4.26\times10^{-6}$\\
ARMA(1,1)   &27.16 & 45.10&0.5312 &0.0374\\
ARMA(2,2)   &28.50 & 55.19&1.0656 & 0.0404\\
ARMA(3,3)   &29.83 & 67.46&1.6957 & 0.0441\\
RNN         &25.50 & 42.06&93.76  &0.0341\\
CNNM        &25.03 & 31.54 &0.0969 & 0.0181\\\hline
LbCNNM      &\textbf{22.34} &\textbf{26.97}&1.3966 &0.2140 \\\bottomrule[2pt]
\end{tabular}\vspace{-0.1in}
\end{center}
\end{table}
\begin{table*}
\caption{Comparison results on M4 (100,000 series).}\label{tb:m4:category}\vspace{-0.2in}
\begin{center}
\renewcommand\tabcolsep{5pt}
\begin{tabular}{c|c|c|c|c|c|c|cc}\toprule[2pt]
        &Hourly     &Daily      &Weekly      &Monthly     &Quarterly      &Yearly & \multicolumn{2}{|c}{overall}\\\cline{2-9}
Methods  &NRMSE     &NRMSE      &NRMSE      &NRMSE     &NRMSE      &NRMSE &sMAPE      &NRMSE\\\hline
Naive   &45.94      &\textbf{3.76}          &11.07      &19.00       &13.89      &18.71      &14.21      &17.15\\
Average &37.09      &5.00       &13.44      &19.89       &16.24      &25.13      &16.89      &19.64\\
Drift   &46.99      &4.57       &14.37      &24.35       &18.08      &20.74      &17.76      &21.24\\
LSR     &49.07      &5.22       &13.14      &21.32       &17.67      &21.95      &17.09      &19.99\\
ExpS    &44.71      &3.77       &11.11      &18.52       &13.74      &18.80      &14.02      &16.89\\
CNNM    &17.84      &5.83       &13.02      &20.39       &18.02      &29.76      &19.77      &21.32\\\hline
LbCNNM  &\textbf{15.61} &4.62 &\textbf{10.55} &\textbf{16.07}  &\textbf{13.33} &\textbf{17.14} &\textbf{13.24} &\textbf{15.15}\\\bottomrule[2pt]
\end{tabular}\vspace{-0.1in}
\end{center}
\end{table*}

\begin{table*}
\caption{Comparison under different levels of stationarity, in terms of NRMSE. }\label{tb:m4:cn}\vspace{-0.2in}
\begin{center}
\renewcommand\tabcolsep{5pt}
\begin{tabular}{cc|ccccccc}\toprule[2pt]
stationarity  &\# series    &Naive      &Average    &Drift      &LSR        &ExpS       &CNNM       &LbCNNM \\\hline
$95\sim100$   &7421         &39.34      &38.80      &52.18      &45.33      &37.90      &35.94      &\textbf{32.96}\\
$80\sim95$    &8056         &32.34      &33.91      &45.37      &40.94      &31.35      &32.15      &\textbf{28.16}\\
$60\sim80$    &11,610       &26.09      &28.58      &36.05      &33.61      &25.63      &26.87      &\textbf{22.58}\\
$40\sim60$    &13,651       &16.96      &18.70      &22.57      &21.05      &16.71      &19.59      &\textbf{15.35}\\
$20\sim40$    &11,891       &11.23      &12.75      &14.11      &13.91      &11.17      &14.42      &\textbf{11.15}\\
$0\sim20$     &47,371       &10.43      &13.99      &10.05      &10.33      &10.48      &18.06      &\textbf{9.27}\\\bottomrule[2pt]
\end{tabular}\vspace{-0.1in}
\end{center}
\end{table*}
\subsection{Forecasting with Incomplete Training Data}
Like CNNM, the LbCNNM program in~\eqref{eq:lbcnnm:noisy2} can also naturally cope with the cases where a portion of the historical observations are missing. However, the model selection techniques in Section~\ref{sec:modelselection} and the data argumentation skills in Section~\ref{sec:argumentation} need be updated to be consistent with the setup of incomplete training data. For simplicity, we consider here only a fixed model size, $m=5h$, and we simply use the generation matrix $G_0$ as the data for learning the transform matrix $A$. In this case, for the computational procedures of LbCNNM to adapt to the setup of incomplete training sequence, we just need to replace PCP by its variant, named Compressive PCP (CPCP)~\cite{cpcp:2012:isit}:
\begin{align}\label{eq:pcpmiss}
\min_{L,S}\|L\|_* + \lambda_{\mathrm{cpcp}}\|S\|_1, \quad\textrm{s.t.}\quad{}\mathcal{P}_{\tilde{\Omega}}(Y - L - S)=0,
\end{align}
where $\tilde{\Omega}\in[1:m]\times[1:n]$ is a 2D sampling set consisting of locations of the observed entries in $Y$ ($Y=G_0$), and $\lambda_{\mathrm{cpcp}}=1/\sqrt{\max(m,n)}$ as usual. Note that, once $L$ and $S$ have been found, we can replace the partial matrix $Y$ by the completion given by $L+S$, thereby the rest procedures in Algorithm~\ref{alg:learnA} remain the same.
\begin{figure}[h!]
\begin{center}
\includegraphics[width=0.48\textwidth]{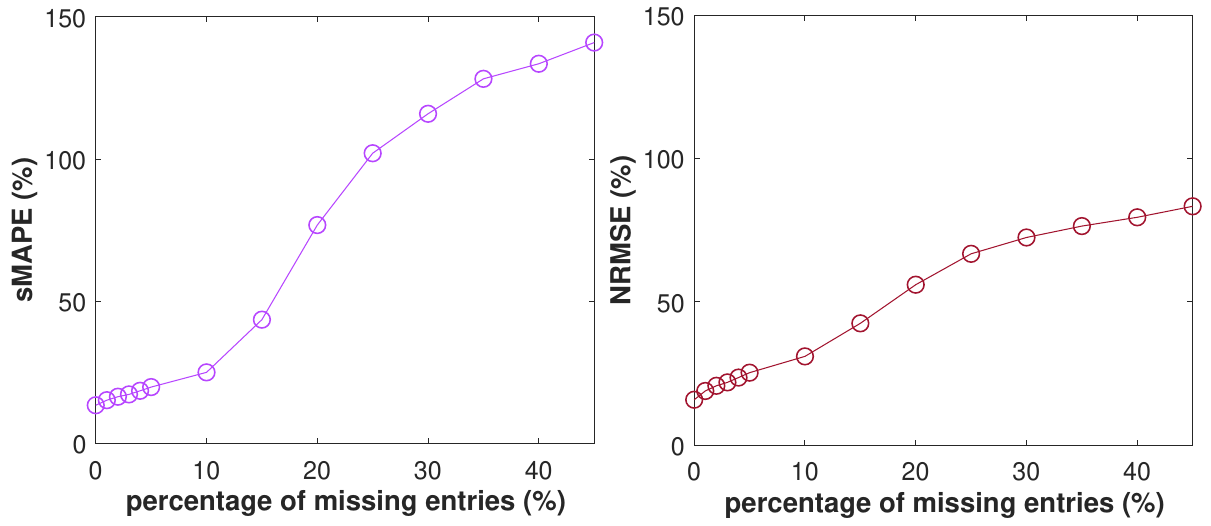}\vspace{-0.15in}
\caption{Prediction error rates of LbCNNM under different missing rates.}\label{fig:miss}\vspace{-0.2in}
\end{center}
\end{figure}

We use the 414 Hourly series in M4 as the experimental data. The locations of the missing entries are sampled randomly, and we run 10 trails for each series. Figure~\ref{fig:miss} plots the averaged error rate as a function of the percentage of missing entries. As one can see, the forecasting accuracy of LbCNNM monotonically declines as the missing rate increases, and drops very fast whenever the percentage of missing entries exceeds 10\%.\footnote{While calculating the prediction error, the observations missing at random are not included. Usually, the error rate given by NRMSE is larger than sMAPE. But Figure~\ref{fig:miss} shows that the sMAPE curve goes up much faster than NRMSE. This is because, in the situation of incomplete training data, LbCNNM tends to produce underestimates, whose error rates are amplified by sMAPE.} These results are not inconsistent with the well-known doctrine that a low-rank matrix can be restored from few, possibly corrupted observations. As aforementioned, the recovery of the low-rank component $L$ is not enough for LbCNNM to provide accurate forecasts, and the sparse component $S$ is also important. Unfortunately, it is hard, if not impossible, to exactly recover a sparse matrix from a subset of its entries.
\begin{table}[h!]
\caption{Results on 20 multivariate times series selected from TSDL.}\label{tb:tsdl:mul}\vspace{-0.25in}
\begin{center}
\begin{tabular}{c|cc|cc}
\toprule[2pt]
             &\multicolumn{2}{c|}{NRMSE} &\multicolumn{2}{c}{Computational Time (secs)} \\\cline{2-5}
  Methods     &mean & std &  Training & Testing\\\hline
LbCNNM(uni)   &21.82 & 32.66 &\textbf{1.66} &\textbf{0.53}\\
LbCNNM(mul)  &\textbf{17.74} & \textbf{22.17} & 4.62 & 3.53\\\bottomrule[2pt]
\end{tabular}\vspace{-0.3in}
\end{center}
\end{table}
\subsection{Multivariate Time Series Forecasting}
While dealing with a multivariate time series, we can implement LbCNNM in two different ways: One is to separate the multivariate series into multiple univariate series and apply LbCNNM individually; this method is referred to as LbCNNM(uni), the other is to reshape the multivariate series (which is a matrix) into a long univariate series (which is a vector) and deploy LbCNNM in a vector completion manner; this is referred to as LbCNNM(mul). To conduct LbCNNM(mul), we only need to update the construction procedure of the generation matrix. Namely, given a $d$-variate training sequence $\tilde{Y}\in\mathbb{R}^{l\times{}d}$ with length $l$, its generation matrix is constructed as follows:
\begin{align*}
&G_0 = [\mathrm{vec}(\tilde{Y}_1),\mathrm{vec}(\tilde{Y}_2),\cdots,\mathrm{vec}(\tilde{Y}_{n_0})]\in\mathbb{R}^{dm\times{}n_0},
\end{align*}
where $n_0=l - m + 1$, $\mathrm{vec}(\cdot)$ is the vectorization operator that reshapes a matrix into a vector, and $\tilde{Y}_i\in\mathbb{R}^{m\times{}d}$ is a submatrix of $\tilde{Y}$ obtained by selecting from $\tilde{Y}$ the rows with indices ranging from $i$ to $i + m -1$.

There are 38 multivariate times series in TSDL, and we use 20 series longer than $11h$ for experiments---so the training part is longer than $10h$. The transform matrix $A$ is learnt by Algorithm~\ref{alg:learnA}, using $Y=G_0$ as the input. We set $m=5h$ and $k=0.5q$, thereby $k=m$ for LbCNNM(uni) and $k=dm$ for LbCNNM(mul). The results presented in Table~\ref{tb:tsdl:mul} show that LbCNNM(mul) outperforms LbCNNM(uni) in terms of NRMSE. However, the computational cost of LbCNNM(mul) is much higher than LbCNNM(uni), especially during the testing stage.
\section{Conclusion}\label{sec:conclusion}
Vector completion with arbitrary sampling, which is to restore a vector from a subset set of its entries chosen arbitrarily, is of considerable significance, and it can generally include the important problem of time series forecasting as a special case. Provided that the target series one wishes to recover is convolutionally low-rank, the recently established CNNM~\cite{liu:arxiv:2019} is provably accurate. However, the required condition, convolutional low-rankness, may not be obeyed whenever the target series is brimful of trends and dynamics. To eliminate these limits, we proposed in this paper a learnable extension to CNNM, termed LbCNNM, that leverages a learnable, orthonormal transformation to assign the series not convolutionally low-rank with the nice property of convolutional low-rankness. Whenever the transformed signal is convolutionally low-rank (resp. approximately so), LbCNNM guarantees to recover the target in an exact (resp. approximate) manner. This drives us to look for adequate methods that can learn transform matrices that comply with training data. Based on PCP~\cite{Candes:2009:JournalACM}, we established a simple yet effective method for learning the desired transformation, and it is arguably that our learning method handles well the trends, seasonality and dynamics of time series. What is more, depending on how the training data is prepared, LbCNNM can also be used as a general tool for model combination, enabling the possibility of providing accurate forecasts on various time series. Finally, we verified the effectiveness of LbCNNM on 100,452 real-world time series from TSDL~\cite{TSDL:dataset} and M4~\cite{M4:ijf:2020}.
\section*{Acknowledgement}
This work is supported by New Generation AI Major Project of Ministry of Science and Technology of China under Grant 2018AAA0100601.
\bibliographystyle{IEEEtran}
\bibliography{lbref}
\begin{IEEEbiography}{Guangcan Liu}(M'11-SM'17) is currently a professor with the School of Automation, Southeast University, Nanjing, China. He received the bachelor's degree in mathematics and the Ph.D. degree in computer science and engineering from Shanghai Jiao Tong University, Shanghai, China, in 2004 and 2010, respectively. He was a Post-Doctoral Researcher with the National University of Singapore, Singapore, from 2011 to 2012, the University of Illinois at Urbana-Champaign, Champaign, IL, USA, from 2012 to 2013, Cornell University, Ithaca, NY, USA, from 2013 to 2014, and Rutgers University, Piscataway, NJ, USA, in 2014. He was a professor with the School of Automation, Nanjing University of Information Science and Technology, Nanjing, China, from 2014 to 2021. His research interests touch on the areas of machine learning, computer vision and signal processing. He is a senior member of the IEEE.
\end{IEEEbiography}

\end{document}